\documentclass[final,12pt]{arxiv}

\usepackage{times}
\usepackage{mathtools}
\usepackage[mathscr]{euscript}
\usepackage{thmtools}
\usepackage{thm-restate}
\usepackage{booktabs}  
\usepackage{caption}
\usepackage{graphicx}
\usepackage{wrapfig}
\usepackage{multirow}

\usepackage{bm}

\let\vec\undefined
\usepackage{MnSymbol}
\DeclareMathAlphabet\mathbb{U}{msb}{m}{n}
\usepackage{xpatch}

\def\Rset{\mathbb{R}}

\DeclareMathOperator*{\E}{\mathbb E}
\DeclareMathOperator*{\argmax}{argmax}
\DeclareMathOperator*{\argmin}{argmin}

\DeclarePairedDelimiter{\abs}{\lvert}{\rvert} 
\DeclarePairedDelimiter{\bracket}{[}{]}
\DeclarePairedDelimiter{\curl}{\{}{\}}

\DeclarePairedDelimiter{\paren}{(}{)}

\hypersetup{
  breaklinks   = true, 
  colorlinks   = true, 
  urlcolor     = blue, 
  linkcolor    = blue, 
  citecolor    = blue 
}

\newcommand{\sC}{{\mathscr C}}
\newcommand{\sD}{{\mathscr D}}
\newcommand{\sE}{{\mathscr E}}
\newcommand{\sF}{{\mathscr F}}

\newcommand{\sH}{{\mathscr H}}

\newcommand{\sM}{{\mathscr M}}

\newcommand{\sX}{{\mathscr X}}
\newcommand{\sY}{{\mathscr Y}}

\newcommand{\sfL}{{\mathsf L}}

\newcommand{\ov}{\overline}
\newcommand{\wt}{\widetilde}

\newcommand{\ignore}[1]{}

\newcommand{\hh}{{\sf h}}

\usepackage[toc, page, header]{appendix}
\newcommand{\ldef}{{\sfL_{\rm{def}}}}

\newcommand{\lsc}{{\sfL}}

\newcommand{\expert}{{g}}
\newcommand{\expertexpert}{{\sf g}}
\newcommand{\sur}{\sfL}

\title[Realizable Learning to Defer]{Realizable $\sH$-Consistent and Bayes-Consistent Loss Functions for Learning to Defer}

\arxivauthor{%
 \Name{Anqi Mao} \Email{aqmao@cims.nyu.edu}\\
 \addr Courant Institute of Mathematical Sciences, New York%
 \AND
 \Name{Mehryar Mohri} \Email{mohri@google.com}\\
 \addr Google Research and Courant Institute of Mathematical Sciences, New York%
 \AND
 \Name{Yutao Zhong} \Email{yutao@cims.nyu.edu}\\
 \addr Courant Institute of Mathematical Sciences, New York%
}

\begin{document}

\maketitle

\begin{abstract}

We present a comprehensive study of surrogate loss functions for
learning to defer. We introduce a broad family of surrogate losses,
parameterized by a non-increasing function $\Psi$, and establish their
realizable $\sH$-consistency under mild conditions. For cost functions
based on classification error, we further show that these losses admit
$\sH$-consistency bounds when the hypothesis set is symmetric and
complete, a property satisfied by common neural network and linear
function hypothesis sets. Our results also resolve an open question
raised in previous work \citep{pmlr-v206-mozannar23a} by proving the
realizable $\sH$-consistency and Bayes-consistency of a specific
surrogate loss. Furthermore, we identify choices of $\Psi$ that lead
to $\sH$-consistent surrogate losses for \emph{any general cost
function}, thus achieving Bayes-consistency, realizable
$\sH$-consistency, and $\sH$-consistency bounds
\emph{simultaneously}. We also investigate the relationship between
$\sH$-consistency bounds and realizable $\sH$-consistency in learning
to defer, highlighting key differences from standard
classification. Finally, we empirically evaluate our proposed
surrogate losses and compare them with existing baselines.

\end{abstract}



\section{Introduction}
\label{sec:intro}

In many practical scenarios, combining expert insights with
established models can yield significant enhancements. These experts
can be human domain specialists or more complex, albeit
resource-intensive, models. For example, modern language and dialogue
models are prone to producing \emph{hallucinations} or inaccurate
information.  The quality of their responses can be significantly
enhanced by delegating uncertain predictions to more specialized or
advanced pre-trained models. This problem is particularly crucial for
large language models (LLMs), as noted in \citep{WeiEtAl2022,
  bubeck2023sparks}. The same principle applies to other generative
systems, like those for images or videos, and to learning models in
diverse applications such as image classification, annotation, and
speech recognition.  Thus, the task of \emph{learning to defer} (L2D)
with experts has become increasingly critical across a wide array of
applications.

Directly optimizing the deferral loss function, which is the target
loss in L2D, is
computationally intractable for many choices of the hypothesis
set. Therefore, a common approach is to optimize a surrogate loss that
facilitates the optimization of the deferral loss function.  Recent
work in L2D has proposed several surrogate losses
\citep{mozannar2020consistent,verma2022calibrated,pmlr-v206-mozannar23a,
  MaoMohriZhong2024deferral} and studied their consistency guarantees,
including Bayes-consistency, realizable $\sH$-consistency, and
$\sH$-consistency bounds (see definitions in
Section~\ref{sec:pre-consistency}). In particular,
\citet{mozannar2020consistent} proposed the first Bayes-consistent
surrogate loss by generalizing the cross-entropy loss for
L2D. \citet{verma2022calibrated} proposed an alternative
Bayes-consistent surrogate loss by generalizing the one-versus-all
loss for L2D. \citet{pmlr-v206-mozannar23a} showed that these
surrogate losses are not realizable $\sH$-consistent. They proposed an
alternative surrogate loss that is realizable $\sH$-consistent, but
they were unable to prove or disprove whether the proposed surrogate
loss is Bayes-consistent. All the surrogate losses mentioned above and
their consistency guarantees hold only for cost functions based on
classification error.  \citet{MaoMohriZhong2024deferral} generalized
the surrogate loss in \citep{mozannar2020consistent} to incorporate
general cost functions and any multi-class surrogate losses. They
provided $\sH$-consistency bounds for the novel family of surrogate
losses, offering a stronger guarantee than Bayes-consistency.

However, none of these surrogate losses satisfies all these guarantees
simultaneously. In particular, a recent AISTATS notable award paper by
\citet{pmlr-v206-mozannar23a} left open the problem of finding
surrogate losses that are both Bayes-consistent and realizable
$\sH$-consistent when the cost function for the expert is its
classification error. The problem becomes even more challenging when
considering more general and realistic cost functions.

We present a comprehensive analysis of surrogate loss functions for
L2D. Our contributions address the limitations of
previous approaches and provide a unified framework for designing
surrogate losses with strong theoretical guarantees.
In Section~\ref{sec:sur}, we first introduce a broad family of
surrogate losses for L2D, derived from first principles
(Section~\ref{sec:derivation}). This family is parameterized by a
non-increasing function $\Psi$, which provides some flexibility in
tailoring the loss function to specific requirements. We establish
that under mild conditions on $\Psi$, these surrogate losses achieve
realizable $\sH$-consistency, a key guarantee for many applications
(Section~\ref{sec:re-H}).

Next, for cost functions based on classification error, we further
establish that our surrogate loss functions admit $\sH$-consistency
bounds when the hypothesis set is symmetric and complete
(Section~\ref{sec:H-bounds}). This result holds for commonly used
neural network and linear function hypothesis sets, further
strengthening the applicability of our results. Additionally, our
results resolve an open question raised by
\citet{pmlr-v206-mozannar23a} by proving the realizable
$\sH$-consistency and Bayes-consistency of their proposed surrogate
loss, which the authors had left as an open question
(Section~\ref{sec:bayes-consistency}).

In Section~\ref{sec:H-bounds}, we further identify specific choices of
$\Psi$, such as the one corresponding to the mean absolute error loss,
that lead to $\sH$-consistent surrogate losses for \emph{any general
cost function}.  These loss functions are adapted to general cost
functions and benefit from Bayes-consistency
(Section~\ref{sec:bayes-consistency}), realizable $\sH$-consistency,
and $\sH$-consistency bounds \emph{simultaneously}.

In Section~\ref{sec:relation}, we also study the relationship between
$\sH$-consistency bounds and realizable $\sH$-consistency in the
context of L2D, highlighting key distinctions from the
standard classification setting.
Finally, we further report the results of experiments with our new
surrogate losses and their comparison with the baselines in different
settings (Section~\ref{sec:experiments}).

We discuss the related work in Section~\ref{sec:related-work} and then
begin with the preliminaries in Section~\ref{sec:pre}.

\section{Related work}
\label{sec:related-work}

The approach of \emph{single-stage learning to defer}, where a
predictor and a deferral function are trained together, was pioneered
by \citet*{CortesDeSalvoMohri2016,CortesDeSalvoMohri2016bis,
  CortesDeSalvoMohri2023} and further developed in subsequent studies
on abstention, where the cost is constant
\citep{charoenphakdee2021classification,
  caogeneralizing,li2023no,cheng2023regression,
  MaoMohriZhong2024score,MaoMohriZhong2024predictor,
  MohriAndorChoiCollinsMaoZhong2024learning} and on deferral, where
the cost can vary depending on the instance and the label
\citep{mozannar2020consistent,verma2022calibrated,
  pmlr-v206-mozannar23a,verma2023learning,
  cao2023defense,MaoMohriMohriZhong2023two,
  MaoMohriZhong2024deferral}. In this approach, the deferral function
determines whether to defer to an expert for each input. This approach
has been shown to be superior to \emph{confidence-based} approaches,
where the decision to abstain or defer is based solely on the
magnitude of the predictor’s value
\citep{Chow1957,chow1970optimum,bartlett2008classification,
  yuan2010classification,WegkampYuan2011, ramaswamy2018consistent,
  NiCHS19,jitkrittum2023does}; and to \emph{selective classification}
approaches, where the selection rate is fixed and a cost function
modeled by an expert cannot be taken into account
\citep{el2010foundations,el2012active,wiener2011agnostic,
  wiener2012pointwise,wiener2015agnostic,geifman2017selective,
  geifman2019selectivenet,acar2020budget,gangrade2021selective,
  zaoui2020regression,jiang2020risk,shah2022selective}.

\citet{madras2018learning} initiated the \emph{learning to defer}
(L2D) problem scenario, which integrates human expert decisions into
the cost function. This approach has been further explored in
subsequent studies \citep{raghu2019algorithmic,wilder2021learning,
  pradier2021preferential}. \citet{mozannar2020consistent} introduced
the first Bayes-consistent surrogate loss for L2D, which was further
refined in
\citep{raman2021improving,liu2022incorporating}. \citet{verma2022calibrated}
proposed an alternative Bayes-consistent surrogate loss, the
one-versus-all loss, which was later examined within a broader family
of loss functions \citep{charusaie2022sample}. \citet{cao2023defense}
proposed an asymmetric softmax function, which can induce a valid
probability estimator for learning to
defer. \cite{pmlr-v206-mozannar23a} showed that the surrogate losses
in \citep{mozannar2020consistent,verma2022calibrated} are not
realizable $\sH$-consistent. They proposed an alternative surrogate
loss that is realizable $\sH$-consistent, but they were unable to
prove or disprove whether the proposed surrogate loss is
Bayes-consistent. All the surrogate losses mentioned above and their
consistency guarantees hold only for cost functions based on
classification error.  \citet{MaoMohriZhong2024deferral} generalized
the surrogate loss in \citep{mozannar2020consistent} to incorporate
general cost functions and any multi-class surrogate losses. They
provided $\sH$-consistency bounds for the novel family of surrogate
losses, offering a stronger guarantee than Bayes-consistency.

Additional studies have focused on post-hoc methods, with
\citet{okati2021differentiable} suggesting an alternative optimization
technique between the predictor and rejector,
\citet{narasimhanpost} offering corrections for underfitting surrogate
losses \citep{liu2024mitigating}, and \citet{Mohammad2024} providing a unifying post-processing framework for multi-objective L2D based on a generalization of the Neyman-Pearson Lemma \citep{neyman1933ix}.  The L2D framework or variations
thereof have found applications in diverse scenarios, spanning
regression, reinforcement learning, and human-in-the-loop systems,
among others
\citep{de2020regression,de2021classification,straitouri2021reinforcement,
  zhao2021directing,joshi2021pre,gao2021human,mozannar2022teaching,
  hemmer2023learning,chen2024learning,palomba2024causal}. More recently, the problem of
\emph{learning to defer with multiple experts} has been analyzed in
several publications
\citep{hemmer2022forming,keswani2021towards,kerrigan2021combining,
  straitouri2022provably,benz2022counterfactual,verma2023learning,
  MaoMohriMohriZhong2023two,MaoMohriZhong2024deferral,
  mao2024regression,tailor2024learning}. Meanwhile,
\citet{MaoMohriMohriZhong2023two} also proposed a \emph{two-stage
learning to defer} framework. They introduced two-stage surrogate
losses that are both Bayes-consistent and realizable $\sH$-consistent
with constant costs. However, realizable $\sH$-consistency does not
hold for cost functions based on classification error. As with
\citep{mozannar2020consistent,verma2022calibrated,pmlr-v206-mozannar23a},
our work focuses on the single-stage and single-expert setting, and we
plan to explore a similar approach in a multi-expert/two-stage setting
in the future.

\newpage
\section{Preliminaries}
\label{sec:pre}

We start with the definitions and notations used in the
learning-to-defer scenario considered in this paper. We will then
introduce consistency guarantees, including \emph{Bayes consistency},
\emph{Realizable $\sH$-consistency}, and \emph{$\sH$-consistency
bounds}. Finally, we will review existing consistent surrogate losses
for L2D.

\subsection{Learning to defer: problem setup}
\label{sec:pre-defer}

Let $\sX$ be an input space and $\sY = [n] := \curl*{1, \ldots, n}$ be
the label space in the standard multi-class classification setting. We
study the \emph{learning to defer} (L2D) scenario, where a learner can
either predict a label from $\sY$ or defer to an expert.

To model this, we introduce an augmented label space $\ov \sY =
\curl*{1, \ldots, n, n + 1}$, where the label $n + 1$ corresponds to
deferral. An expert is a fixed predictor $\expert \colon \sX \times
\sY \to \Rset$.  The goal of L2D is to select a predictor $h$ out of a
hypothesis set $\sH$ of functions mapping from $\sX \times \ov \sY$ to
$\Rset$ with small expected \emph{deferral loss}.
Let $\hh(x)$ denote the prediction of $h$ on input $x \in \sX$,
defined as $\hh(x) = \argmax_{y \in \ov \sY} h(x, y)$, that is the
label in the augmented label space $\ov \sY$ with the highest score,
with an arbitrary but fixed deterministic strategy for breaking ties.
Then, the \emph{deferral loss function} $\ldef$ is defined as
follows:
\begin{equation*}
\forall (x, y) \in \sX \times \sY,
\quad \ldef(h, x, y) 
= 1_{\hh(x) \neq y} 1_{\hh(x) \in [n]} + c(x, y) 1_{\hh(x) = n + 1},
\end{equation*}
where $c(x, y)$ is the the cost of deferring on input $x$
with true label $y$.
If the deferral option is selected, that is $\hh(x) = n + 1$, the
deferral cost $c(x, y)$ is incurred.  Otherwise, the prediction of
$h$ is within the standard label space, $\hh(x) \in [n]$, and the
loss incurred coincides with the standard zero-one classification
loss, $1_{\hh(x) \neq y}$.

The choice of the cost function $c$ is
flexible. For example, the cost can be defined as the expert's
classification error: $c(x, y) = 1_{\expertexpert(x) \neq y}$, as in
previous work
\citep{mozannar2020consistent,verma2022calibrated,pmlr-v206-mozannar23a}.
Here, $\expertexpert(x) = \argmax_{y\in [n]}\expert(x,y)$ is the
prediction made by the expert $\expert$.  More generally, it can
incorporate the inference cost for the expert
\citep{MaoMohriZhong2024deferral}: $c(x, y) = \alpha
1_{\expertexpert(x) \neq y} + \beta$, with $\alpha, \beta > 0$. We
assume, without loss of generality, that the cost is bounded by 1: $0
\leq c(x, y) \leq 1 $, which can be achieved through normalization in
practice.

\subsection{Consistency guarantees}
\label{sec:pre-consistency}

Directly optimizing the deferral loss function, which is the target
loss in L2D, is generally computationally intractable for for complex 
hypothesis sets $\sH$. Therefore, a common approach is
to optimize a surrogate loss that facilitates the optimization of the
deferral loss function. A natural learning guarantee for such
surrogate losses is \emph{Bayes-consistency}
\citep{Zhang2003,bartlett2006convexity,zhang2004statistical,tewari2007consistency,steinwart2007compare}:
\begin{definition}[Bayes-consistency]
A surrogate loss $\sfL$ is Bayes-consistent with respect to $\ldef$, if minimizing the surrogate loss over the family of all measurable functions leads to the minimization of the deferral loss:
\begin{equation*}
\lim_{n \to \plus \infty } \sE_{\sfL} (h_n) - \sE^*_{\sfL}(\sH_{\rm{all}}) = 0 \implies \lim_{n \to \plus \infty } \sE_{\ldef}(h_n) - \sE^*_{\ldef}(\sH_{\rm{all}}) = 0.
\end{equation*}
\end{definition}
Here, given a distribution $\sD$ over $\sX \times \sY$ and a loss
function $\sfL \colon \sH \times \sX \times \sY \to \Rset$, we denote
by $\sE_{\sfL} (h)$ the \emph{generalization error} of a hypothesis $h
\in \sH$, $\sE_{\lsc}(h) = \E_{(x, y) \sim \sD}[\lsc(h, x, y)]$, and
by $\sE^*_{\sfL}(\sH)$ the \emph{best-in-class generalization error},
$\sE^*_{\lsc}(\sH) = \inf_{h \in \sH}
\sE_{\lsc}(h)$. Bayes-consistency assumes that the optimization occurs
over the family of all measurable functions,
$\sH_{\rm{all}}$. However, in practice, the hypothesis set of interest
is typically a restricted one, such as a family of neural
networks. Therefore, a hypothesis-dependent learning guarantee, such
as \emph{$\sH$-consistency bounds}
\citep{awasthi2022Hconsistency,awasthi2022multi} (see also
\citep{AwasthiMaoMohriZhong2023theoretically,awasthi2023dc,MaoMohriZhong2023characterization,MaoMohriZhong2023structured,mao2023cross,zheng2023revisiting,MaoMohriZhong2023ranking,MaoMohriZhong2023rankingabs,mao2024universal,mao2024top,mao2024h,cortes2024cardinality}) and \emph{realizable
$\sH$-consistency} \citep{long2013consistency,zhang2020bayes}, is more
informative and relevant. Realizable $\sH$-consistency, defined as
follows, requires that a minimizer of the surrogate loss over the
given hypothesis set $\sH$ also minimizes the target loss, provided
that the underlying distribution is realizable.
\begin{definition}[Realizable $\sH$-consistency]
\label{def:realizable}
A surrogate loss $\sfL$ is realizable $\sH$-consistent with respect to
$\ldef$, if for any distribution over which there exists a 
predictor $h^* \in \sH$ achieving zero deferral loss, $\sE_{\ldef}(h^*) = 0$, 
minimizing the surrogate loss also leads to a zero-error solution:
\begin{equation*}
\hat h \in \argmin_{h \in \sH} \sE_{\sfL}(h) \implies \sE_{\ldef}(\hat h) = 0.
\end{equation*}
\end{definition}
Note that realizable $\sH$-consistency does not imply
Bayes-consistency, even if we set $\sH = \sH_{\rm{all}}$ in
Definition~\ref{def:realizable}, since Bayes-consistency requires that
the relationship holds for all distributions, not just realizable
ones. \emph{$\sH$-consistency bounds}, on the other hand, always imply
Bayes-consistency.  Given a hypothesis set $\sH$, a surrogate loss
$\sfL$ admits an \emph{$\sH$-consistency bound}, if for some
non-decreasing concave function $\Gamma \colon \Rset_{+}\to \Rset_{+}$ with
$\Gamma(0) = 0$, a bound of the following form holds for any
hypothesis $h\in \sH$ and any distribution:
\begin{align}
\label{eq:est-bound}
\sE_{\ldef}(h) - \sE^*_{\ldef}(\sH)+\sM_{\ldef}(\sH)
\leq \Gamma\paren*{\sE_{\sfL}(h) - \sE^*_{\sfL}(\sH)+\sM_{\sfL}(\sH}),
\end{align}
where $\sM_{\sfL}(\sH)$ is \emph{the minimizability gap}, defined as
the difference between the best-in-class generalization error and the
expected pointwise infimum loss: $\sM_{\sfL}(\sH) = \sE^*_{\sfL}(\sH)
- \mathbb{E}_{x} \bracket* {\inf_{h\in \sH}\E_{y|x}\bracket*{\sfL(h,
    x, y)}}$. The minimizability gap can be upper-bounded by the
approximation error and vanishes when $\sH = \sH_{\rm{all}}$
\citep{awasthi2022Hconsistency,awasthi2022multi}. Thus, an
$\sH$-consistency bound implies Bayes-consistency. The relationship
between the two hypothesis-dependent learning guarantees---realizable
$\sH$-consistency and $\sH$-consistency bounds---depends on the target
loss adopted in the specific learning scenario. In
Section~\ref{sec:relation}, we will demonstrate that in the standard
multi-class classification setting, an $\sH$-consistency bound is a
stronger notion than realizable $\sH$-consistency. However, in
L2D, these guarantees do not imply one another.

\subsection{Existing surrogate losses}
\label{sec:pre-exist-sur}

Here, we will review several consistent surrogate losses used in
L2D. For convenience, we use $\wt c(x, y) =
1_{\expertexpert(x) \neq y}$ to denote the cost when it specifically
represents the expert's classification error, and use $c(x, y)$ when
it represents a general cost function.

\citet{mozannar2020consistent} proposed the first Bayes-consistent
surrogate loss by generalizing the cross-entropy loss for L2D, with
cost functions based on classification error, which is defined as
\begin{equation*}
  \sfL_{\rm{CE}}(h, x, y)
  = - \log \paren*{\frac{e^{h(x, y)}}{\sum_{y' \in \ov \sY} e^{h(x, y')}}}
  - (1 - \wt c(x, y)) \log \paren*{\frac{e^{h(x, n + 1)}}{\sum_{y' \in \ov \sY} e^{h(x, y')}}}.
\end{equation*}
\citet{verma2022calibrated} proposed an alternative one-vs-all
surrogates loss with cost functions based on expert's classification
error, that is Bayes-consistent as well:
\begin{align*}
\sfL_{\rm{OvA}}(h, x, y) 
& = \Phi (h(x, y)) \!+\! \sum_{\substack{y' \in \ov \sY\\ y'\neq y}} \Phi (- h(x, y'))
+ \paren*{1 - \wt c(x, y)} \bracket*{\Phi(h(x, n + 1)) \!-\! \Phi(-h(x, n + 1))},
\end{align*}
where $\Phi$ is a strictly proper binary composite loss
\citep{reid2010composite}, such as the logistic loss $t \mapsto \log(1 + e^{-t})$.
$\sfL_{\rm{CE}}$ and $\sfL_{\rm{OvA}}$ are not realizable
$\sH$-consistent. Instead, \citet{pmlr-v206-mozannar23a} proposed the following loss function
that is realizable $\sH$-consistent when $\sH$ is closed under
scaling:
\begin{equation*}
\sfL_{\rm{RS}}(h, x, y) = - 2 \log \paren*{ \frac{e^{h(x, y)} + (1 - \wt c(x, y)) e^{h(x, n + 1)}}{\sum_{y' \in \ov \sY} e^{h(x, y')}} }.
\end{equation*}
However, they were unable to prove or disprove whether the surrogate loss $\sfL_{\rm{RS}}$ is Bayes-consistent.

All the surrogate losses mentioned above and their consistency
guarantees hold only for cost functions based on the classification
error: $\wt c(x, y) = 1_{\expertexpert(x) \neq
  y}$. \citet{MaoMohriZhong2024deferral} generalized the surrogate
loss $\sfL_{\rm{CE}}$ to incorporate general cost functions and any
multi-class surrogate losses:
\begin{equation*}
\sfL_{\rm{general}}(h, x, y) = \ell(h, x, y) + (1 - c(x, y)) \ell(h, x, n + 1).
\end{equation*}
Here, $\ell$ is a Bayes-consistent surrogate loss for the multi-class
zero-one loss over the augmented label set $\ov \sY$. In particular,
$\ell$ can be chosen as a comp-sum loss \citep{mao2023cross}, for
example, the generalized cross entropy loss (see Section~\ref{sec:derivation}). As
shown by \citet{MaoMohriZhong2024deferral}, $\sfL_{\rm{general}}$
benefits from $\sH$-consistency bounds, which implies its
Bayes-consistency.

\section{Novel surrogate losses}
\label{sec:sur}

In this section, we introduce a new family of surrogate losses for
L2D that benefit from Bayes-consistency, realizable
$\sH$-consistency and $\sH$-consistency bounds, starting from first
principles.

\subsection{Derivation from first principles}
\label{sec:derivation}

Observe that for any $(x, y) \in \sX \times \sY$, we have 
$1_{\hh(x) = n + 1} = 1_{\hh(x) \neq y} 1_{\hh(x) = n + 1}$, since
$\hh(x) = n + 1$ implies $\hh(x) \neq y$.
Thus, using additionally $1_{\hh(x) \in [n]} = 1_{\hh(x) \neq n + 1}$, 
the deferral loss can be rewritten as follows 
for all $(x, y) \in \sX \times \sY$:
\begin{align}
\ldef(h, x, y) 
& = 1_{\hh(x) \neq y} 1_{\hh(x) \in [n]} + c(x, y) 1_{\hh(x) = n + 1} \nonumber \\
& = 1_{\hh(x) \neq y} 1_{\hh(x) \neq n + 1} + c(x, y) 1_{\hh(x) \neq y} 1_{\hh(x) = n + 1} \nonumber \\
& = 1_{\hh(x) \neq y} 1_{\hh(x) \neq n + 1} + c(x, y) 1_{\hh(x) \neq y} (1 - 1_{\hh(x) \neq n + 1}) \nonumber \\
& = c(x, y) 1_{\hh(x) \neq y} + \paren*{1 - c(x, y)} 1_{\hh(x) \neq y \wedge \hh(x) \neq n + 1}.
\label{eq:def}
\end{align}
\ignore{
\begin{align}
\ldef(h, x, y) 
& = 1_{\hh(x) \neq y} 1_{\hh(x) \in [n]} + c(x, y) 1_{\hh(x) = n + 1} \nonumber \\
& = c(x, y) 1_{\hh(x) \neq y} + \paren*{1 - c(x, y)} 1_{\hh(x) \neq y} 1_{\hh(x) \neq n + 1} \nonumber \\
& = c(x, y) 1_{\hh(x) \neq y} + \paren*{1 - c(x, y)} 1_{\hh(x) \neq y \wedge \hh(x) \neq n + 1},
\label{eq:def}
\end{align}
where the second equality is derived by analyzing three cases: $\hh(x) = y$, $\hh(x) = n + 1$ and $\hh(x) \neq y, \hh(x) \neq n + 1$ and the third inequality
by using $1_{A} 1_{B} = 1_{A \cap B}$. Thus, in particular, 
when $c(x, y) = 1$, the loss incurred
is the standard $1_{\hh(x) \neq y}$ and, when $c(x, y) = 0$,  the loss
is $1_{\hh(x) \neq y \wedge \hh(x) \neq n + 1} =
1_{\hh(x) \neq y \wedge \hh(x) \in [n]}$.
}

\ignore{
Intuitively, this means that when $c(x, y) = 1$, indicating that the expert incurs an error, the model learns to predict the true label ($\hh(x) = y$) to incur zero loss. When $c(x, y) = 0$, meaning the expert is accurate, the model learns to either predict the true label ($\hh(x) = y$) or defer the prediction to the expert ($\hh(x) = n + 1$), both of which result in zero loss.
}

Next, we will derive the new surrogate losses for L2D by replacing the indicator functions in \eqref{eq:def} with smooth loss functions. The first indicator function $1_{\hh(x) \neq y}$ is just the multi-class zero-one loss. Thus, a natural choice is to replace it with a surrogate loss in standard multi-class classification. We will specifically consider the family of comp-sum losses \citep{mao2023cross}, defined as follows for any $(h, x, y) \in \sH \times \sX \times \sY$:
\begin{equation*}
\ell_{\rm{comp}}(h, x, y) = \Psi \paren*{\frac{e^{h(x, y)}}{\sum_{y' \in \ov \sY} e^{h(x, y')}}},
\end{equation*}
where $\Psi \colon [0, 1] \to \Rset_{+} \cup
\curl{+\infty}$ is a non-increasing function. For example, by taking $\Psi(t) = -\log(t)$, $\frac{1}{q} \paren*{1 - t^{q}} \text{ with } q\in (0, 1)$, $1 - t$, we obtain the \emph{logistic loss}
\citep{Verhulst1838,Verhulst1845,Berkson1944,Berkson1951},\ignore{the
\emph{sum exponential loss} \citep{weston1998multi,awasthi2022multi},}
the \emph{generalized cross entropy loss}
\citep{zhang2018generalized}, and the \emph{mean absolute error loss}
\citep{ghosh2017robust}, respectively:
\begin{alignat*}{2}
& \text{Logistic loss: } & \ell_{\rm{log}}(h, x, y) & = 
- \log \bracket*{\frac{e^{h(x, y)}}{\sum_{y' \in \ov \sY} e^{h(x, y')}}}\\
& \text{Generalized cross entropy loss: } 
& \ell_{\rm{gce}}(h, x, y) & = \frac{1}{q}\bracket*{1 - \bracket*{\frac{e^{h(x,y)}}
    {\sum_{y'\in \ov \sY} e^{h(x,y')}}}^{q}}\\
& \text{Mean absolute error loss: } 
& \ell_{\rm{mae}}(h, x, y) & =  1 - \frac{e^{h(x, y)}}{\sum_{y' \in \ov \sY} e^{h(x, y')}}.
\end{alignat*}
For any $(h, x, y) \in \sH \times \sX \times \ov \sY$, the confidence
margin $\rho_h(x, y)$ is defined by $\rho_h(x, y) = h(x, y) - \max_{y' \in \ov \sY,
  y' \neq y} h(x, y')$.  Thus, the second indicator function
$1_{\hh(x) \neq y \wedge \hh(x) \neq n + 1}$ can be expressed as
follows in terms of the confidence margin:
\begin{align*}
1_{\hh(x) \neq y \wedge \hh(x) \neq n + 1} 
&= 1_{\paren*{h(x, y) \leq \max_{y' \in \ov \sY, y' \neq y} h(x, y')} \wedge \paren*{h(x, n + 1) \leq \max_{y' \in \ov \sY, y' \neq n + 1} h(x, y')}}\\
& = 1_{\paren*{\rho_h(x, y) \leq 0} \wedge \paren*{\rho_h(x, n + 1) \leq 0}}\\
& = 1_{\max \curl*{\rho_h(x, y), \rho_h(x, n + 1)} \leq 0}.
\end{align*}
Note that the first indicator function can also be written in terms of
margin: $1_{\hh(x) \neq y} = 1_{\rho_h(x, y) \leq 0}$.  Unlike the
first indicator function, which presses $h(x, y)$ to be the largest
score among $\ov \sY$, that is the margin $\rho_h(x, y)$ to be
positive, the second indicator function only enforces $h(x, y)$ or
$h(x, n + 1)$ to be the largest score among $\ov \sY$, that is the
maximum of two margins, $\max \curl*{\rho_h(x, y), \rho_h(x, n + 1)}$,
to be positive. This condition can be further strengthened by
requiring the sum of two margins, $\rho_h(x, y) + \rho_h(x, n + 1)$,
to be positive. In view of this observation, we adopt the following
modified comp-sum surrogate loss for the second indicator function:
\begin{equation*}
  \wt \ell_{\rm{comp}}(h, x, y)
  = \Psi \paren*{\frac{e^{h(x, y)} + e^{h(x, n + 1)}}
    {\sum_{y' \in \ov \sY} e^{h(x, y')}}},
\end{equation*}
where $\Psi \colon [0, 1] \to \Rset_{+} \cup \curl{+\infty}$ is a
non-increasing function.  In other words, $\wt \ell_{\rm{comp}}$
replaces the term $e^{h(x, y)}$ in the softmax function in
$\ell_{\rm{comp}}$ with the sum $e^{h(x, y)} + e^{h(x, n + 1)}$. The
effect is to encourage the sum of the two margins, $\rho_h(x, y) +
\rho_h(x, n + 1)$, to be positive, rather than just the single margin
$\rho_h(x, y)$. Following this principle, we derive the following
expression for a new family of surrogate losses, $\sfL_{\rm{RL2D}}$, dubbed \emph{realizable L2D}:
\begin{equation}
\label{eq:surrogae-new}
\sfL_{\rm{RL2D}}(h, x, y)
= c(x, y) \ell_{\rm{comp}}(h, x, y)
+ \paren*{1 - c(x, y)}  \wt \ell_{\rm{comp}}(h, x, y).
\end{equation}
For the choices of $\Psi(t) = -\log(t)$, $\frac{1}{q} \paren*{1 -
  t^{q}} \text{ with } q\in (0, 1)$ and $1 - t$, we obtain the new
surrogate losses for L2D in Table~\ref{tab:comp}. In the
next sections, we will prove both realizable $\sH$-consistency
guarantees and $\sH$-consistency bounds for this family of surrogate
losses, which imply their excess error bounds and Bayes-consistency as
well.

\begin{table}[t]
\caption{A new family of surrogate losses $\sfL_{\rm{RL2D}}$ for
  L2D.}
\label{tab:comp}
\centering
\begin{tabular}{@{\hspace{0cm}}lll@{\hspace{0cm}}}
  \toprule
  $\Psi(t)$ & $\sfL_{\rm{RL2D}}$\\
  \midrule
  $-\log(t)$ & $- c(x, y)\log \bracket*{\frac{e^{h(x, y)}}{\sum_{y' \in \ov \sY} e^{h(x, y')}}} - \paren*{1 - c(x, y)} \log \bracket*{\frac{e^{h(x, y)} + e^{h(x, n + 1)}}{\sum_{y' \in \ov \sY} e^{h(x, y')}}}$ \\
  $\frac{1}{q} \paren*{1 - t^q}$ & $ \frac{c(x, y)}{q}\bracket*{1 - \bracket*{\frac{e^{h(x,y)}}
    {\sum_{y'\in \ov \sY} e^{h(x,y')}}}^{q}} + \frac{\paren*{1 - c(x, y)} }{q}\bracket*{1 - \bracket*{\frac{e^{h(x,y)} + e^{h(x, n + 1)}}
    {\sum_{y'\in \ov \sY} e^{h(x,y')}}}^{q}}$
     \\
    $1 - t$ & $c(x, y)\paren*{1 - \frac{e^{h(x, y)}}{\sum_{y' \in \ov \sY} e^{h(x, y')}}} +  \paren*{1 - c(x, y)} \paren*{1 - \frac{e^{h(x, y)} + e^{h(x, n + 1)}}{\sum_{y' \in \ov \sY} e^{h(x, y')}}}$  \\
    \bottomrule
  \end{tabular}
\end{table}

\subsection{Realizable $\sH$-consistency}
\label{sec:re-H}

Here, we show that $\sfL_{\rm{RL2D}}$ is realizable $\sH$-consistent
with respect to $\ldef$. We say that a hypothesis set $\sH$ is
\emph{closed under scaling} if, $h \in \sH \implies \alpha h \in \sH$
for any $\alpha \in \Rset$.
\begin{restatable}{theorem}{RealizableConsistencyNew}
\label{Thm:re-consistency-new}
Assume that $\sH$ is closed under scaling. Suppose that $\Psi$ is
non-increasing, $\Psi \paren*{\frac{2}{3}} > 0$ and $\lim_{t \to 1}
\Psi(t) = 0$. Then, the surrogate loss $\sfL_{\rm{RL2D}}$ is
realizable $\sH$-consistent with respect to $\ldef$.
\end{restatable}
The proof, detailed in Appendix~\ref{app:re-consistency-new}, begins
by establishing an upper bound on the deferral loss in terms of the
comp-sum loss: $\ldef \leq
\frac{\sfL_{\rm{RL2D}}}{\Psi\paren*{\frac{2}{3}}}$. Letting $\hat h$
be the minimizer of $\sfL_{\rm{RL2D}}$ and $\alpha$ be any real
number, we then show that $\sE_{\ldef}(\hat h) \leq
\frac{1}{\Psi\paren*{\frac{2}{3}}} \sE_{\sfL_{\rm{RL2D}}}(\alpha
h^*)$. The generalization error is then split by conditioning on
whether $h^*(x)$ is the deferral class $(n + 1)$ or not.  Finally,
we demonstrate that each conditional term converges to zero as
$\alpha$ tends to $+\infty$, and apply the monotone convergence
theorem to complete the proof.

\subsection{$\sH$-Consistency bounds}
\label{sec:H-bounds}

Here, we show that $\sfL_{\rm{RL2D}}$ admits an $\sH$-consistency
bound with respect to $\ldef$, which implies its Bayes-consistency as
well. We say that a hypothesis set is symmetric if there exists a
family $\sF$ of functions $f$ mapping from $\sX$ to $\Rset$ such that
\[\curl*{\bracket*{h(x, 1), \ldots, h(x, n + 1)} \colon h \in \sH} =
\curl*{\bracket*{f_1(x),\ldots, f_{n + 1}(x)} \colon f_1, \ldots, f_{n + 1} \in \sF},\] for any $x \in \sX$. We say that a hypothesis set
$\sH$ is complete if for any $(x, y) \in \sX \times \sY$, the set of
scores generated by it spans across the real numbers: $\curl*{h(x, y)
  \mid h \in \sH} = \Rset$. Common neural network and linear
function hypothesis sets are all symmetric and complete. We first consider the case where the cost
is expert's classification error.
\begin{restatable}{theorem}{HConsistencyNew}
\label{Thm:H-consistency-new}
Assume that $\sH$ is symmetric and complete and that $c(x, y) =
1_{\expertexpert(x) \neq y}$. Then, for all $h \in \sH$ and any
distribution, the following $\sH$-consistency bound holds:
\begin{equation*}
  \sE_{\ldef}(h) - \sE_{\ldef}(\sH) + \sM_{\ldef}(\sH)
  \leq \Gamma \paren*{\sE_{\sfL_{\rm{RL2D}}}(h)
    - \sE_{\sfL_{\rm{RL2D}}}(\sH) + \sM_{\sfL_{\rm{RL2D}}}(\sH)},
\end{equation*}
where $\Gamma(t) = \sqrt{2t}$ when $\Psi(t) = -\log(t)$ and $\Gamma(t)
= \sqrt{2 (n + 1)^q t}$ when $\Psi(t) = \frac{1}{q} \paren*{1 - t^{q}}
\text{ with } q\in (0, 1)$.
\end{restatable}
The proof, detailed in Appendix~\ref{app:log} and \ref{app:gce},
establishes strong consistency guarantees for our new surrogate loss
$\sfL_{\rm{RL2D}}$ (Theorem~\ref{Thm:H-consistency-new}).  We first
introduce $y_{\max} = \argmax_{y \in \sY} p(x, y)$, the label with the
highest conditional probability. We then show that for any hypothesis
$h$ and input $x$, if $y_{\max}$ is not the predicted label
$h_{\max}$, the conditional error of $h$ is lower bounded by a
modified hypothesis $\ov h$ (obtained by swapping the scores of
$y_{\max}$ and $h_{\max}$). Next, for hypotheses where $y_{\max} =
h_{\max}$, we lower bound their conditional regret in terms of the
conditional regret of the deferral loss using a new hypothesis
$h_\mu$. This proof is novel and significantly different from existing
approaches for establishing $\sH$-consistency bounds in either the
standard or deferral settings \citep{mao2023cross,
  MaoMohriZhong2024deferral}.

The next result further shows that when $\Psi(t) = 1 - t$, our
surrogate losses benefit from $\sH$-consistency bounds for any general
cost function.

\begin{restatable}{theorem}{HConsistencyMAE}
\label{Thm:H-consistency-mae}
Assume that $\sH$ is symmetric and complete. Suppose that $\Psi(t) = 1
- t$. Then, for all $h \in \sH$ and any distribution, the following
$\sH$-consistency bounds hold:
\begin{equation*}
\sE_{\ldef}(h) - \sE_{\ldef}(\sH) + \sM_{\ldef}(\sH) \leq (n + 1) \paren*{\sE_{\sfL_{\rm{RL2D}}}(h) - \sE_{\sfL_{\rm{RL2D}}}(\sH) + \sM_{\sfL_{\rm{RL2D}}}(\sH)}.
\end{equation*}
\end{restatable}
The proof is included in
Appendix~\ref{app:mae}. Theorem~\ref{Thm:H-consistency-new}
provides stronger consistency guarantees for our new surrogate loss
$\sfL_{\rm{RL2D}}$ with $\Psi(t) = 1 - t$ since it holds for any
general cost function. The proof idea is similar to that of
Theorem~\ref{Thm:H-consistency-new}, albeit with more cases to analyze
due to the general cost function. This occurs when lower bounding the
conditional regret of a hypothesis $h$, which satisfies $y_{\max} =
h_{\max}$, in terms of the conditional regret of the deferral loss
by introducing a new hypothesis $h_{\mu}$. The additional cases
necessitate a more stringent condition for the guarantee, such that
the functions $\Psi(t) = -\log(t)$ and $\Psi(t) = \frac{1}{q} \left(1
- t^q\right)$ do not apply.


\subsection{Excess error bounds and Bayes-consistency}
\label{sec:bayes-consistency}

For the family of all measurable functions $\sH = \sH_{\rm{all}}$, the
minimizability gaps vanish. In this case,
Theorems~\ref{Thm:H-consistency-new} and \ref{Thm:H-consistency-mae}
imply the following excess error bounds and Bayes-consistency
guarantees.

\begin{restatable}{corollary}{ExcessBoundNew}
\label{Cor:excess-bound-new}
Suppose that $c(x, y) = 1_{\expertexpert(x) \neq y}$. For all $h \in
\sH_{\rm{all}}$ and any distribution, the following excess error
bounds hold:
\begin{equation*}
  \sE_{\ldef}(h) - \sE_{\ldef}(\sH_{\rm{all}})
  \leq \Gamma \paren*{\sE_{\sfL_{\rm{RL2D}}}(h)
    - \sE_{\sfL_{\rm{RL2D}}}(\sH_{\rm{all}})},
\end{equation*}
where $\Gamma(t) = \sqrt{2t}$ when $\Psi(t) = -\log(t)$ and $\Gamma(t)
= \sqrt{2 (n + 1)^q t}$ when $\Psi(t) = \frac{1}{q} \paren*{1 - t^{q}}
\text{ with } q\in (0, 1)$. Furthermore, the surrogate loss
$\sfL_{\rm{RL2D}}$ is Bayes-consistent with respect to $\ldef$ in
these cases.
\end{restatable}

\begin{restatable}{corollary}{ExcessBoundNewMAE}
\label{Cor:excess-bound-new-mae}
Suppose that $\Psi(t) = 1 - t$. For all $h \in \sH_{\rm{all}}$ and any
distribution, the following excess error bounds hold:
\begin{equation*}
  \sE_{\ldef}(h) - \sE_{\ldef}(\sH_{\rm{all}})
  \leq (n + 1) \paren*{\sE_{\sfL_{\rm{RL2D}}}(h)
    - \sE_{\sfL_{\rm{RL2D}}}(\sH_{\rm{all}})}.
\end{equation*}
Furthermore, the surrogate loss $\sfL_{\rm{RL2D}}$ is Bayes-consistent with respect to $\ldef$ in this case.
\end{restatable}
Therefore, Theorem~\ref{Thm:re-consistency-new} and
Corollary~\ref{Cor:excess-bound-new} show that $\sfL_{\rm{RL2D}}$ is
both realizable $\sH$-consistent and Bayes-consistent with respect to
$\ldef$. This solves the open problem raised by
\citet{pmlr-v206-mozannar23a}.

In particular, for cost functions based on classification error, $c(x,
y) = 1_{\expertexpert(x) \neq y}$, our surrogate loss
$\sfL_{\rm{RL2D}}$ with $\Psi(t) = -\log(t)$ coincides with the
surrogate loss $\sfL_{\rm{RS}}$ in \citep{pmlr-v206-mozannar23a},
modulo a constant.  This affirmatively answers the question of whether
their surrogate loss is Bayes-consistent when $c(x, y) =
1_{\expertexpert(x) \neq y}$. However, their surrogate loss cannot be
shown to be Bayes-consistent for a general cost function. In
contrast, our surrogate losses $\sfL_{\rm{RL2D}}$ with $\Psi(t) = 1
- t$ are adaptable to general cost functions and benefit from both
$\sH$-consistency bounds and realizable $\sH$-consistency
guarantees. We also provide a more general family of comp-sum loss
functions with $\Psi(t) = \frac{1}{q} \left(1 - t^q\right)$ that
benefit from both $\sH$-consistency bounds and realizable
$\sH$-consistency when $c(x, y) = 1_{\expertexpert(x) \neq y}$.

\section{Relationship between $\sH$-consistency bounds
  and realizable $\sH$-consistency}
\label{sec:relation}

Here, we discuss the relationship between $\sH$-consistency bounds and
realizable $\sH$-consistency. First, realizable $\sH$-consistency does
not imply $\sH$-consistency bounds, since $\sH$-consistency bounds
require that the relationship holds for all distributions, not just
realizable ones. Moreover, $\sH$-consistency bounds provide
non-asymptotic guarantees, while realizable $\sH$-consistency provides
only asymptotic guarantees.  Second, $\sH$-consistency bounds imply
realizable $\sH$-consistency in the standard multi-class
classification setting. This is because minimizability gaps vanish
under the realizable assumption in standard case.  In particular, for
comp-sum losses, the following holds (see
Appendix~\ref{app:minimizability} for proof).
\begin{restatable}{theorem}{Minimizability}
\label{Thm:minimizability}
Assume that there exists a zero error solution $h^* \in \sH$ with
$\sE_{\ell_{0-1}}(h^*) = 0$ and $\sH$ is closed under scaling. Assume
that $\lim_{t \to 1} \Psi(t) = 0$. Then, the minimizability gap of
comp-sum loss $\ell_{\rm{comp}}$ vanishes:
$\sM_{\ell_{\rm{comp}}}(\sH) = 0$.
\end{restatable}
However, in the deferral setting, this relationship no longer holds:
$\sH$-consistency bounds cannot imply realizable $\sH$-consistency. In
particular, \citet{mao2023cross} showed that $\sfL_{\rm{CE}}$ benefits
from $\sH$-consistency bounds, while \citet{pmlr-v206-mozannar23a}
showed that it is not realizable $\sH$-consistent. The loss function
in \citep{madras2018learning} is not Bayes-consistent, and thus does
not have $\sH$-consistency bound guarantees, but is actually
realizable $\sH$-consistent \citep{pmlr-v206-mozannar23a}.

\begin{table}[t]
\vskip -0.2in
\caption{Comparison of system accuracy, accepted accuracy and coverage; mean $\pm$ standard deviation over three runs. Realizable L2D outperforms or is comparable to baselines in all the settings.}
    \label{tab:comparison}
\begin{center}
    \resizebox{\textwidth}{!}{
    \begin{tabular}{@{\hspace{0pt}}lllll@{\hspace{0pt}}}
      Method & Dataset & System Accuracy & Accepted Accuracy & Coverage \\
    \midrule
    \citet{mozannar2020consistent} ($\sur_{\rm{CE}}$) & \multirow{6}{*}{HateSpeech} &  $91.60 \pm 0.15$ &  $94.61 \pm 0.67$  &  $44.55 \pm 1.68$ \\
    \citet{verma2022calibrated} ($\sur_{\rm{OvA}}$)  & &  $92.18 \pm 0.10$ & $95.43 \pm 0.36$ & $58.56 \pm 3.18$ \\
    \citet{pmlr-v206-mozannar23a} ($\sur_{\rm{RS}}$) & & 91.83 $\pm$ 0.63 & 95.37 $\pm$ 0.72 & 54.78 $\pm$ 3.70 \\
    \citet{MaoMohriZhong2024deferral} ($\sur_{\rm{general}}$) & & $92.05 \pm 0.04$ & $96.28 \pm 0.35$ & $46.74 \pm 2.80$\\
    Realizable L2D ($\sfL_{\rm{RL2D}}, q = 0.7$) &  & \underline{92.20 $\pm$ 0.54} & 96.06 $\pm$ 0.39 & 57.85 $\pm$ 0.76 \\
    Realizable L2D ($\sfL_{\rm{RL2D}}, q = 1$) &  & $91.97 \pm 0.29$ &  $96.57 \pm 0.69$  &  $53.25 \pm 2.49$  \\
    \midrule
    \citet{mozannar2020consistent} ($\sur_{\rm{CE}}$) & \multirow{6}{*}{COMPASS} & $66.33 \pm 0.47$ & $73.65 \pm 1.83$ & $55.17 \pm 9.51$\\
    \citet{verma2022calibrated} ($\sur_{\rm{OvA}}$)  & & $66.33 \pm 1.31$ & $71.03 \pm 5.10$ & $53.33 \pm 4.73$ \\
    \citet{pmlr-v206-mozannar23a} ($\sur_{\rm{RS}}$) & & $66.00 \pm 2.27$  & $63.20 \pm 4.23$ & $69.50 \pm 10.8$ \\
    \citet{MaoMohriZhong2024deferral} ($\sur_{\rm{general}}$) & & $66.67 \pm 0.62$ & $76.25 \pm 2.42$ & $48.33 \pm 5.31$ \\
    Realizable L2D ($\sfL_{\rm{RL2D}}, q = 0.7$) &  & $66.17 \pm 2.01$ &  $69.33 \pm 3.03$ & $55.67 \pm 5.95$ \\
    Realizable L2D ($\sfL_{\rm{RL2D}}, q = 1$) &  & \underline{66.83 $\pm$ 0.85} &  69.02 $\pm$ 2.42  &  54.83 $\pm$ 0.62  \\
    \midrule
    \citet{mozannar2020consistent} ($\sur_{\rm{CE}}$) & \multirow{6}{*}{CIFAR-10H} & 96.27 $\pm$ 0.51 & 98.77 $\pm$ 0.71 & 64.33 $\pm$ 6.13\\
    \citet{verma2022calibrated} ($\sur_{\rm{OvA}}$)  & & 96.25 $\pm$ 0.45 & 98.74 $\pm$ 0.54 & 67.88 $\pm$ 6.16 \\
    \citet{pmlr-v206-mozannar23a} ($\sur_{\rm{RS}}$) & & 96.63 $\pm$ 0.18  & 98.23 $\pm$ 0.78 & 66.63 $\pm$ 1.80 \\
    \citet{MaoMohriZhong2024deferral} ($\sur_{\rm{general}}$) & & 96.75 $\pm$ 0.55 & 98.65 $\pm$ 0.80 & 65.68 $\pm$ 3.36 \\
    Realizable L2D ($\sfL_{\rm{RL2D}}, q = 0.7$) &  & \underline{96.80 $\pm$ 0.25} &  98.37 $\pm$ 0.20 & 76.77 $\pm$ 3.63 \\
    Realizable L2D ($\sfL_{\rm{RL2D}}, q = 1$) &  & 96.57 $\pm$ 0.05 &  98.34 $\pm$ 0.24  &  77.37 $\pm$ 2.43  \\
    \end{tabular}
    }
\end{center}
    \vskip -0.3in
\end{table}

\section{Experiments}
\label{sec:experiments}

In this section, we empirically evaluate our proposed
surrogate losses and compare them with existing baselines.

\textbf{Experimental settings.} We follow the setting of \citet{pmlr-v206-mozannar23a} and conduct experiments on \ignore{ a synthetic dataset: Mixture-of-Gaussians \citep{pmlr-v206-mozannar23a}, and} three real-world datasets: CIFAR-10H \citep{battleday2020capturing}, HateSpeech \citep{davidson2017automated}, and COMPASS \citep{dressel2018accuracy}. 
\ignore{For the Mixture-of-Gaussians, we adopt the exact realizable setting from \citep[Section~7.2]{pmlr-v206-mozannar23a}, which is realizable by linear functions: there exists a linear hypothesis $h^* \in \sH$ achieving zero deferral loss, $\sE_{\ldef}(h^*) = 0$.
}
For these three datasets, we adopt the same model class as that in \citep[Table~1]{pmlr-v206-mozannar23a}. Each dataset is randomly split into 70\%, 10\%, and 20\% for training, validation, and testing, respectively.

As with \citep{pmlr-v206-mozannar23a}, we choose the cost function to be the expert's
classification error: $c(x, y) = 1_{\expertexpert(x) \neq y}$.
We compare our surrogate to four baselines as described in Section~\ref{sec:pre-exist-sur}: the cross-entropy surrogate $\sfL_{\rm{CE}}$ from \citep{mozannar2020consistent}, the one-vs-all surrogate  $\sfL_{\rm{OvA}}$ from \citep{mozannar2020consistent}, the realizable surrogate $\sfL_{\rm{RS}}$ from \citep{pmlr-v206-mozannar23a}, and the general surrogate $\sfL_{\rm{general}}$ from \citep{MaoMohriZhong2024deferral}. For $\sfL_{\rm{OvA}}$, we choose $\Phi$ as the logistic loss, following \citep{verma2022calibrated}. For $\sfL_{\rm{general}}$, we choose $\ell$ as the generalized cross entropy loss with $q = 0.7$, following \citep{MaoMohriZhong2024deferral}. For our Realizable L2D surrogate $\sfL_{\rm{RL2D}}$, we consider two choices: $\ell$ as the generalized cross entropy loss with $q = 0.7$, following \citep{zhang2018generalized,MaoMohriZhong2024deferral}, and $\ell$ as the mean absolute error loss ($q = 1$).
Among these, $\sfL_{\rm{CE}}$, $\sfL_{\rm{OvA}}$ and $\sfL_{\rm{general}}$ are Bayes-consistent but not realizable $\sH$-consistent; $\sfL_{\rm{RS}}$, $\sfL_{\rm{RL2D}}$ with $q = 0.7$ and $\sfL_{\rm{RL2D}}$ with $q = 1$ are both Bayes-consistent and realizable $\sH$-consistent, as shown in Sections~\ref{sec:re-H} and \ref{sec:bayes-consistency}. Note that in this case, $\sfL_{\rm{RS}}$ is a special case of $\sfL_{\rm{RL2D}}$ when $\Psi$ is chosen as $t \mapsto -\log(t)$. We use the same optimizer, learning rate, and number of epochs as chosen in \citep{pmlr-v206-mozannar23a}, and we select the model that achieves the highest \emph{system accuracy}, that is average $\bracket*{1 - \ldef(h, x, y)}$, on a validation set.

\textbf{Evaluation.} For the three real-world datasets, we report the
\emph{system accuracy}, that is average value of $\bracket*{1 -
\ldef(h, x, y)}$ on the test data.
For completeness, we also include the \emph{accepted accuracy}, that
is the average value of $\bracket*{1_{\hh(x) \neq y} 1_{\hh(x) \in
[n]}}$.  This metric considers only incorrect predictions ($\hh(x)
\neq y$) and measures the fraction of those where the system's output
($\hh(x)$) falls within the valid range of possible outputs ($[n]$).
We also report the \emph{coverage}, that is the average value of
$\bracket*{1_{\hh(x) \in [n]}}$ on the test set, or the fraction of
test instances where the system's prediction falls within the valid
range ($[n]$).
For each metric, we average results over three runs and report the
mean accuracy along with the standard deviation for both our proposed
methods and the baseline approaches.
\ignore{
For the realizable Mixture-of-Gaussians, we plot the system accuracy of various methods on a held-out test dataset consisting of 5,000 points as we increase the size of the training data.
}

\textbf{Results.} 
\ignore{
Figures 1 show that for the realizable synthetic dataset, $\sfL_{\rm{RS}}$, $\sfL_{\rm{RL2D}}$ with $q = 0.7$, and $\sfL_{\rm{RL2D}}$ with $q = 1$ are able to achieve a near zero-error solution, while $\sfL_{\rm{CE}}$, $\sfL_{\rm{OvA}}$, and $\sfL_{\rm{general}}$ fail to find a near zero-error solution. This verifies our realizable $\sH$-consistency results in Section~\ref{sec:re-H}.
}
Table~\ref{tab:comparison} shows that for the real-world datasets, $\sfL_{\rm{RL2D}}$ with $q = 0.7$, and $\sfL_{\rm{RL2D}}$ with $q = 1$ either outperform or are comparable to the best baseline in terms of system accuracy on each dataset. This performance is supported by our $\sH$-consistency bounds and Bayes-consistency results for our Realizable L2D surrogate with respect to the deferral loss $\ldef$, as shown in Sections~\ref{sec:H-bounds} and \ref{sec:bayes-consistency}. Table~\ref{tab:comparison} also shows that $\sfL_{\rm{RL2D}}$ achieves reasonable coverage and acceptable accuracy. The system accuracy, coverage, and standard deviations of the baselines match those in \citep{pmlr-v206-mozannar23a}. Moreover, $\sfL_{\rm{RS}}$, $\sfL_{\rm{RL2D}}$ with $q = 0.7$, and $\sfL_{\rm{RL2D}}$ with $q = 1$ perform differently across various datasets: $\sfL_{\rm{RL2D}}$ with $q = 0.7$ outperforms the others on HateSpeech and CIFAR-10H, while $\sfL_{\rm{RL2D}}$ with $q = 1$ outperforms the others on COMPASS. Note that in this case, $\sfL_{\rm{RS}}$ is a special case of $\sfL_{\rm{RL2D}}$ when $\Psi$ is chosen as $t \mapsto -\log(t)$. These results show that Realizable L2D can benefit from the flexibility in the choice of $\Psi$.

\ignore{
\begin{figure}[t]
\vskip -0.1in
\begin{center}
\includegraphics[scale=0.5]{figs/mix_of_guassians.pdf}
\caption{Comparison of system accuracy on the realizable synthetic dataset.}
\label{fig:comparison}
\end{center}
\vskip -0.1in
\end{figure} 
}

\section{Conclusion}
\label{sec:conclusion}

We introduced a broad family of surrogate losses and algorithms for
learning to defer, parameterized by a non-increasing function. We
established their realizable $\sH$-consistency properties under mild
conditions and proved that several of these surrogate losses benefit
from $\sH$-consistency bounds for cost functions based on
classification error and general cost functions, which also imply
their Bayes-consistency.  This research not only resolves an open
question posed in previous work but also lays the groundwork for
comparing various consistency notions in learning to defer and
standard classification. Looking forward, our approach offers a
promising avenue for analyzing multi-expert and two-stage settings.


\bibliography{rdef}

\newpage
\appendix

\renewcommand{\contentsname}{Contents of Appendix}
\tableofcontents
\addtocontents{toc}{\protect\setcounter{tocdepth}{4}} 
\clearpage

\section{Proof of realizable \texorpdfstring{$\sH$}{H}-consistency}
\label{app:re-consistency-new}

\RealizableConsistencyNew*
\begin{proof}
We first prove that for every $(h, x, y) \in \sH \times \sX \times \sY$, the following inequality holds:
\begin{equation*}
\ldef(h, x , y) \leq \frac{\sfL_{\rm{RL2D}}(h, x, y)}{\Psi\paren*{\frac{2}{3}}}.
\end{equation*}
We will analyze case by case.
\begin{enumerate}
    \item \textbf{Case I}: If $\hh(x) \in [n]$ (deferral does not occur):
    \begin{enumerate}
        \item If $1_{\hh(x) \neq y} = 1$, then we must have 
        \begin{align*}
        & \ldef(h, x , y) = 1, \quad \frac{e^{h(x, y)}}{\sum_{y' \in \ov \sY} e^{h(x, y')}} \leq \frac12, \quad \frac{e^{h(x, y)} + e^{h(x, n + 1)}}{\sum_{y' \in \ov \sY} e^{h(x, y')}} \leq \frac23\\
        & \implies \sfL_{\rm{RL2D}}(h, x, y) \geq c(x, y) \Psi\paren*{\frac12} + \paren*{1 - c(x, y)} \Psi\paren*{\frac{2}{3}} \geq \Psi\paren*{\frac{2}{3}}\, \ldef(h, x , y).
        \end{align*}
        \item If $1_{\hh(x) \neq y} = 0$, then we must have
        \begin{equation*}
            \sfL_{\rm{RL2D}}(h, x, y) \geq 0 = \ldef(h, x , y).
        \end{equation*}
    \end{enumerate}
    \item \textbf{Case II}: If $\hh(x) = n + 1$ (deferral occurs): then we must have
    \begin{align*}
        & \ldef(h, x , y) = c(x, y), \quad
        \frac{e^{h(x, y)}}{\sum_{y' \in \ov \sY} e^{h(x, y')}} \leq \frac12\\
        & \implies \sfL_{\rm{RL2D}}(h, x, y) \geq  c(x, y) \Psi\paren*{\frac12} \geq \Psi\paren*{\frac{2}{3}}\, \ldef(h, x , y).
    \end{align*}
\end{enumerate}
This concludes that $\ldef(h, x , y) \leq \frac{\sfL_{\rm{RL2D}}(h, x, y)}{\Psi\paren*{\frac{2}{3}}} $. Next, we prove that $\sfL_{\rm{RL2D}}$ is realizable $\sH$-consistent under the assumptions. Consider a distribution and an expert under which there exists a zero error solution $h^* \in \sH$ with $\sE_{\ldef}(h^*) = 0$. Let $\hat h$ be the minimizer of the surrogate loss: $\hat h \in \argmin_{h \in \sH} \sE_{\sfL_{\rm{RL2D}}}(h)$. Let $\alpha$ be any real number. Then, the following inequality holds:
\begin{align*}
\sE_{\ldef}(\hat h) 
& \leq \frac{1}{\Psi\paren*{\frac{2}{3}}} \sE_{\sfL_{\rm{RL2D}}}(\hat h) \tag{$\ldef \leq \frac{1}{\Psi\paren*{\frac{2}{3}}}\sfL_{\rm{RL2D}}$}\\
& \leq \frac{1}{\Psi\paren*{\frac{2}{3}}} \sE_{\sfL_{\rm{RL2D}}}(\alpha h^*) \tag{$\hat h \in \argmin_{h \in \sH} \sE_{\sfL_{\rm{RL2D}}}(h)$ and $\sH$ is closed under scaling}\\
& = \frac{1}{\Psi\paren*{\frac{2}{3}}} \E \bracket*{\sfL_{\rm{RL2D}}(\alpha h^*, x, y) \mid \hh^*(x) = n + 1} \mathbb{P} \paren*{\hh^*(x) = n + 1}\\
& \qquad + \frac{1}{\Psi\paren*{\frac{2}{3}}} \E \bracket*{\sfL_{\rm{RL2D}}(\alpha h^*, x, y) \mid \hh^*(x) \in [n]} \mathbb{P} \paren*{\hh^*(x) \in [n]}.
\end{align*}
For the first term conditional on $\hh^*(x) = n + 1$, we must have $h^*(x, n + 1) > \max_{y \in \sY} h^*(x, y)$ and $c(x, y) = 0$ since the data is realizable. Therefore,
\begin{align*}
& \lim_{\alpha \to \plus \infty } \E \bracket*{\sfL_{\rm{RL2D}}(\alpha h^*, x, y) \mid \hh^*(x) = n + 1} \mathbb{P} \paren*{\hh^*(x) = n + 1}\\
& = \lim_{\alpha \to \plus \infty } \E \bracket[\Bigg]{\Psi \paren*{\frac{e^{\alpha h^*(x, y)} + e^{\alpha h^*(x, n + 1)}}{\sum_{y' \in \ov \sY} e^{\alpha h^*(x, y')}}}  \mid \hh^*(x) = n + 1} \mathbb{P} \paren*{\hh^*(x) = n + 1}\\
& = \E \bracket*{0 \mid \hh^*(x) = n + 1} \mathbb{P} \paren*{\hh^*(x) = n + 1} \tag{$\lim_{t \to 1} \Psi(t) = 0$ and monotone convergence theorem}\\
& = 0.
\end{align*}
For the second term conditional on $\hh^*(x) \in [n]$, we must have $h^*(x, y) > \max_{y' \in \ov \sY, y' \neq y} h(x, y')$ since the data is realizable. Therefore,
\begin{align*}
& \lim_{\alpha \to \plus \infty } \E \bracket*{\sfL_{\rm{RL2D}}(\alpha h^*, x, y) \mid \hh^*(x) \in [n]} \mathbb{P} \paren*{\hh^*(x) \in [n]}\\
& = \lim_{\alpha \to \plus \infty } \E \bracket[\Bigg]{c(x, y) \Psi \paren*{\frac{e^{\alpha h^*(x, y)}}{\sum_{y' \in \ov \sY} e^{\alpha h^*(x, y')}}}\\
& \qquad + \paren*{1 - c(x, y)}  \Psi \paren*{ \frac{e^{\alpha h^*(x, y)} + e^{\alpha h^*(x, n + 1)}}{\sum_{y' \in \ov \sY} e^{\alpha h^*(x, y')}}}  \mid \hh^*(x) \in [n]} \mathbb{P} \paren*{\hh^*(x) \in [n]}\\
& = \E \bracket*{0 \mid \hh^*(x) \in [n]} \mathbb{P} \paren*{\hh^*(x) \in [n]} \tag{$\lim_{t \to 1} \Psi(t) = 0$ and monotone convergence theorem}\\
& = 0.
\end{align*}
Combining the two analyses, we conclude that $\sE_{\ldef}(\hat h)  = 0$ and thus $\sfL_{\rm{RL2D}}$ is realizable $\sH$-consistent with respect to $\ldef$.
\end{proof}

\section{Proof of \texorpdfstring{$\sH$}{H}-consistency bounds}
\label{app:H-consistency-new}

Before delving into the proof, we first establish some essential notation and definitions. Let $\sur$ represent a deferral surrogate loss and $\sH$ denote a hypothesis set. We define the conditional error as  $\sC_{\sur}(h, x) = \E_{y \mid x} \bracket*{\sur(h, x, y)}$, the best-in-class conditional error as  $\sC^*_{\sur}\paren*{\sH, x} = \inf_{ h \in \sH} \sC_{\sur}(h, x)$, and the conditional regret as $\Delta \sC_{\sur, \sH} \paren*{h, x} = \sC_{\sur}(h, x) - \sC^*_{\sur}\paren*{\sH, x}$. We proceed to present a general theorem demonstrating that, to establish $\sH$-consistency bounds \eqref{eq:est-bound} with a concave function $\Gamma$, it suffices to lower bound the conditional regret of the surrogate loss by that of the deferral loss, using the same function $\Gamma$.

\begin{restatable}
  {theorem}{Tool-Gamma}
\label{Thm:tool-Gamma}
If the following holds for all $h \in \sH$ and $x \in \sX$, for some concave function $\Gamma $:
\begin{equation}
\label{eq:c-regret-Gamma}
 \Delta \sC_{\ldef, \sH}(h, x) \leq \Gamma
\paren*{\Delta \sC_{\sur, \sH}(h, x)},
\end{equation}
then, for all hypotheses $h \in \sH$ and for any distribution,
    \begin{equation*}
     \sE_{\ldef}(h)- \sE^*_{\ldef}(\sH) + \sM_{\ldef}(\sH) \leq \Gamma \paren*{\sE_{\sur}(h) - \sE^*_{\sur}(\sH) + \sM_{\sur}(\sH)}.
    \end{equation*}
\end{restatable}
\begin{proof}
We can express the expectations of the conditional regrets for $\ldef$ and $\sur$ as follows:
\begin{align*}
\E_{x} \bracket*{\Delta \sC_{\ldef, \sH}(h, x)} & = \sE_{\ldef}(h) - \sE^*_{\ldef}(\sH) + \sM_{\ldef}(\sH)\\
\E_{x} \bracket*{\Delta \sC_{\sur, \sH}(h, x)} & = \sE_{\sur}(h) - \sE^*_{\sur}(\sH) + \sM_{\sur}(\sH).
\end{align*}
Then, by using \eqref{eq:c-regret-Gamma} and taking the expectation, we obtain:
\begin{align*}
\sE_{\ldef}(h) - \sE^*_{\ldef}(\sH) + \sM_{\ldef}(\sH) 
& = \E_{x} \bracket*{\Delta \sC_{\ldef, \sH}(h, x)}\\
& \leq \E_{x} \bracket*{ \Gamma
\paren*{\Delta \sC_{\sur, \sH}(h, x)} } \tag{Eq. \eqref{eq:c-regret-Gamma}}\\
& \leq 
\Gamma \paren*{\E_{x} \bracket*{\Delta \sC_{\sur, \sH}(h, x)}} \tag{concavity of $\Gamma$}\\
& = \Gamma \paren*{\sE_{\sur}(h) - \sE^*_{\sur}(\sH) + \sM_{\sur}(\sH)}.
\end{align*}
Thus, the proof is complete.
\end{proof}

Next, to prove $\sH$-consistency bounds using Theorem~\ref{Thm:tool-Gamma}, we will characterize the conditional regret of the deferral loss $\ldef$ in the following section.

\subsection{Auxiliary lemma}

To simplify the presentation, we introduce the following notation. For any $y \in \sY$, define $p(x, y) = \mathbb{P}(Y = y \mid X = x)$ as the conditional probability that $Y = y$ given $X = x$. For brevity, we will omit the dependency on $x$ in our notation.  We denote by $h_y
= h(x, y)$ for any $y \in \ov \sY$. We also denote by $p_y = p(x, y)$
and $q_y = p(x, y) c(x, y)$ for any $y \in \sY$, and $p_{n + 1} =
\sum_{y \in \sY} p(x, y) (1 - c(x, y))$. Note that $p(x, y) \paren*{1
  - c(x, y)} = p_y - q_y$, $ \forall y \in \sY$. Let $p_{\hh} =
p_{\hh(x)} = \begin{cases} p_{\hh(x)} & \hh(x) \in [n]\\ p_{n + 1} & \hh(x) =
  n + 1.
\end{cases}
$. Let $y_{\max} = \argmax_{y \in \sY} p_y$ and $h_{\max} = \argmax_{y
  \in \sY} h_y$. Note that both $y_{\max}$ and $h_{\max}$ are in the
label space $\sY$, while $\hh(x)$ is in the augmented label space $\ov
\sY$. We characterize the conditional regret of the deferral loss $\ldef$ as follows.
\begin{lemma}
\label{lemma:delta_target}
Assume that $\sH$ is symmetric and complete. Then, the conditional regret of the deferral loss $\ldef$ can be expressed as follows:
$
\Delta \sC_{\ldef, \sH}(h, x) = \max \curl*{p_{y_{\max}}, p_{n + 1}} -  p_{\hh}.
$
\end{lemma}
\begin{proof}
We can write the conditional error of the deferral loss as follows:
\begin{align*}
& \sC_{\ldef}(h, x)\\
& = \sum_{y \in \sY} p(x, y) \ldef(h, x, y)\\
& = \sum_{y \in \sY} p(x, y) 1_{\hh(x) \neq y} 1_{\hh(x) \in [n]} + \sum_{y \in \sY} p(x, y) c(x, y) 1_{\hh(x) = n + 1}\\
& = (1 - p_{\hh(x)})1_{\hh(x) \in [n]} + \paren*{1 - p_{n + 1}} 1_{\hh(x) = n + 1}\\
& = 1 - p_{\hh}.
\end{align*}
Since $\sH$ is symmetric and complete, for any $x \in \sX$, $\curl*{\hh(x) \colon h \in \sH} = \ov \sY$. Then, the best-in-class conditional error of $\ldef$ can be expressed as follows: 
\begin{equation}
\sC^*_{\ldef}\paren*{\sH, x} = \inf_{ h \in \sH} \sC_{\ldef}(h, x) = 1 - \max \curl*{p_{n + 1}, p_{y_{\max}}}
\end{equation}
Therefore, $\Delta \sC_{\ldef, \sH}(h, x) = \sC_{\ldef}(h, x) - \sC^*_{\ldef}\paren*{\sH, x} =  \max \curl*{p_{y_{\max}}, p_{n + 1}} -  p_{\hh}$.
\end{proof}

Next, we will present the proofs separately in the following sections, by lower bounding the conditional regret of the surrogate loss $\sur$ by that of the deferral loss $\ldef$ using Lemma~\ref{lemma:delta_target}.

\subsection{\texorpdfstring{$\Psi(t) = 1 -t$}{Psi}}
\label{app:mae}

\begin{restatable}{theorem}{HConsistencyNewMae}
\label{Thm:H-consistency-new-mae}
Assume that $\sH$ is symmetric and complete. Then, for all $h \in \sH$ and any distribution, the following $\sH$-consistency bound holds:
\begin{equation*}
\sE_{\ldef}(h) - \sE_{\ldef}(\sH) + \sM_{\ldef}(\sH) \leq n \paren*{\sE_{\sfL_{\rm{RL2D}}}(h) - \sE_{\sfL_{\rm{RL2D}}}(\sH) + \sM_{\sfL_{\rm{RL2D}}}(\sH)}.
\end{equation*}
\end{restatable}

\begin{proof}
We can write the conditional error of the surrogate loss as follows:
\begin{align*}
& \sC_{\sfL_{\rm{RL2D}}}(h, x)\\
& = \sum_{y \in \sY} p(x, y) \sfL_{\rm{RL2D}}(h, x, y)\\
& = \sum_{y \in \sY} p(x, y) c(x, y) \paren*{1 - \frac{e^{h(x, y)}}{\sum_{y' \in \ov \sY} e^{h(x, y')}}} +  \sum_{y \in \sY} p(x, y) \paren*{1 - c(x, y)}  \paren*{1 - \frac{e^{h(x, y)} + e^{h(x, n + 1)}}{\sum_{y' \in \ov \sY} e^{h(x, y')}}}\\
&= \sum_{y \in \sY} q_y \paren*{1 - \frac{e^{h_y}}{\sum_{y' \in \ov \sY} e^{h_{y'}}}} + \sum_{y \in \sY} \paren*{p_y - q_y} \paren*{1 - \frac{e^{h_y} + e^{h_{n + 1}}}{\sum_{y' \in \ov \sY} e^{h_{y'}}}}.
\end{align*}
By Lemma~\ref{lemma:delta_target}, the conditional regret of the deferral loss can be expressed as
\begin{equation*}
\Delta \sC_{\ldef, \sH}(h, x) = \max \curl*{p_{y_{\max}}, p_{n + 1}} -  p_{\hh}.
\end{equation*}
Next, we will show that the conditional regret of the surrogate loss can be lower bounded as follows:
\begin{equation}
\label{eq:cond-bound-new}
\Delta \sC_{\sfL_{\rm{RL2D}}, \sH}(h, x) = \sC_{\sfL_{\rm{RL2D}}}(h) - \sC^*_{\sfL_{\rm{RL2D}}}(\sH) \geq \frac1{n + 1}
\paren*{\Delta \sC_{\ldef, \sH}(h, x)}.
\end{equation}
We first prove that for any hypothesis $h$ and $x \in \sX$, if $y_{\max} \neq h_{\max}$, then the conditional error of $h$ can be lower bounded by that of $\ov h$, which satisfies that
$\ov h(x, y) = \begin{cases}
h_{h_{\max}} &  y = y_{\max}\\
h_{y_{\max}} & y = h_{\max}\\
h_{y} & \text{otherwise}.
\end{cases}
$. Indeed,
\begin{align*}
\sC_{\sfL_{\rm{RL2D}}}(h) - \sC_{\sfL_{\rm{RL2D}}}(\ov h)
& = q_{y_{\max}} \paren*{1 - \frac{e^{h_{y_{\max}}}}{\sum_{y' \in \ov \sY} e^{h_{y'}}}} + \paren*{p_{y_{\max}} - q_{y_{\max}}} \paren*{1 - \frac{e^{h_{y_{\max}}} + e^{h_{n + 1}}}{\sum_{y' \in \ov \sY} e^{h_{y'}}}} \\
& \qquad +q_{h_{\max}} \paren*{1 - \frac{e^{h_{h_{\max}}}}{\sum_{y' \in \ov \sY} e^{h_{y'}}}} + \paren*{p_{h_{\max}} - q_{h_{\max}}} \paren*{1 - \frac{e^{h_{h_{\max}}} + e^{h_{n + 1}}}{\sum_{y' \in \ov \sY} e^{h_{y'}}}}\\
& \qquad - q_{y_{\max}} \paren*{1 - \frac{e^{h_{h_{\max}}}}{\sum_{y' \in \ov \sY} e^{h_{y'}}}} - \paren*{p_{y_{\max}} - q_{y_{\max}}} \paren*{1 - \frac{e^{h_{h_{\max}}} + e^{h_{n + 1}}}{\sum_{y' \in \ov \sY} e^{h_{y'}}}} \\
& \qquad -q_{h_{\max}} \paren*{1 - \frac{e^{h_{y_{\max}}}}{\sum_{y' \in \ov \sY} e^{h_{y'}}}} - \paren*{p_{h_{\max}} - q_{h_{\max}}} \paren*{1 - \frac{e^{h_{y_{\max}}} + e^{h_{n + 1}}}{\sum_{y' \in \ov \sY} e^{h_{y'}}}}\\
& =\frac{1}{\sum_{y' \in \ov \sY} e^{h_{y'}}} \paren*{ p_{y_{\max}} - p_{h_{\max}}} \paren*{e^{h_{h_{\max}}} - e^{h_{y_{\max}}}}
\geq 0.
\end{align*}
Therefore, we only need to lower bound the conditional regret of hypothesis $h$ satisfying $y_{\max} = h_{\max}$. 
Next, we will analyze case by case. Note that when $(p_{y_{\max}} - p_{n + 1}) (h_{y_{\max}} - h_{n + 1}) > 0$, we have $\Delta \sC_{\ldef, \sH}(h, x) = \max \curl*{p_{y_{\max}}, p_{n + 1}} -  p_{\hh} = 0$.
\begin{enumerate}
    \item \textbf{Case I}: If $p_{y_{\max}} - p_{n + 1} \geq 0$ and $h_{y_{\max}} - h_{n + 1} \leq 0$: we define a new hypothesis $h_{\mu}$ such that 
    $h_{\mu}(x, y) = \begin{cases}
\log \paren*{e^{h_{n + 1}} + \mu} &  y = y_{\max}\\
\log \paren*{e^{h_{y_{\max}}} - \mu} & y = n + 1\\
h(x, y) & \text{otherwise}.
\end{cases}
$, where $e^{h_{y_{\max}}} \geq \mu \geq 0$. Then, we can  lower bound the conditional regret of $\sfL_{\rm{RL2D}}$ by using $\Delta \sC_{\sfL_{\rm{RL2D}}, \sH}(h, x) \geq \sC_{\sfL_{\rm{RL2D}}}(h) - \sC^*_{\sfL_{\rm{RL2D}}}(h_{\mu})$ for any  $e^{h_{y_{\max}}} \geq \mu \geq 0$:
\begin{align*}
& \Delta \sC_{\sfL_{\rm{RL2D}}, \sH}(h, x) \\
& \geq \sup_{e^{h_{y_{\max}}} \geq \mu \geq 0} \paren*{\sC_{\sfL_{\rm{RL2D}}}(h) - \sC^*_{\sfL_{\rm{RL2D}}}(h_{\mu})}\\
& \geq \sup_{e^{h_{y_{\max}}} \geq \mu \geq 0} \paren[\Bigg]{q_{y_{\max}} \paren*{1 - \frac{e^{h_{y_{\max}}}}{\sum_{y' \in \ov \sY} e^{h_{y'}}}} + \paren*{p_{y_{\max}} - q_{y_{\max}}} \paren*{1 - \frac{e^{h_{y_{\max}}} + e^{h_{n + 1}}}{\sum_{y' \in \ov \sY} e^{h_{y'}}}}\\
& \qquad + \sum_{y' \in \sY, y' \neq y_{\max}} \paren*{p_{y'} - q_{y'}} \paren*{1 - \frac{e^{h_{y'}} + e^{h_{n + 1}}}{\sum_{y' \in \ov \sY} e^{h_{y'}}}}\\
& \qquad - q_{y_{\max}} \paren*{1 - \frac{e^{h_{n + 1}} + \mu}{\sum_{y' \in \ov \sY} e^{h_{y'}}}} - \paren*{p_{y_{\max}} - q_{y_{\max}}} \paren*{1 - \frac{e^{h_{n + 1}} + e^{h_{h_{\max}}}}{\sum_{y' \in \ov \sY} e^{h_{y'}}}}} \\
& \qquad - \sum_{y' \in \sY, y' \neq y_{\max}} \paren*{p_{y'} - q_{y'}} \paren*{1 - \frac{e^{h_{y'}} + e^{h_{y_{\max}}} - \mu}{\sum_{y' \in \ov \sY} e^{h_{y'}}}}\\
& = \frac{1}{\sum_{y' \in \ov \sY} e^{h_{y'}}} \sup_{e^{h_{y_{\max}}} \geq \mu \geq 0} \paren*{q_{y_{\max}} \paren*{e^{h_{n + 1}} + \mu - e^{h_{y_{\max}}}} + \paren*{p_{n + 1} - p_{y_{\max}} + q_{y_{\max}}}  \paren*{e^{h_{y_{\max}}} - \mu - e^{h_{n + 1}}}}\\
& = \paren*{p_{y_{\max}} - p_{n + 1}} \frac{e^{h_{n + 1}}}{\sum_{y' \in \ov \sY} e^{h_{y'}}} \tag{$\mu = e^{h_{y_{\max}}}$ achieves the maximum}\\
&\geq \frac1{n + 1} \paren*{p_{y_{\max}} - p_{n + 1}}
\tag{by the assumption $ h_{n + 1} \geq h_{y_{\max}} = h_{h_{\max}}$}\\
& = \frac1{n + 1} \paren*{\Delta \sC_{\ldef, \sH}(h, x)} \tag{by the assumption $p_{y_{\max}} \geq p_{n + 1}$ and $h_{y_{\max}} - h_{n + 1} \leq 0$}
\end{align*}
\item \textbf{Case II}: If $p_{y_{\max}} - p_{n + 1} \leq 0$ and $h_{y_{\max}} - h_{n + 1} \geq 0$: we define a new hypothesis $h_{\mu}$ such that 
    $h_{\mu}(x, y) = \begin{cases}
\log \paren*{e^{h_{n + 1}} - \mu} &  y = y_{\max}\\
\log \paren*{e^{h_{y_{\max}}} + \mu} & y = n + 1\\
h(x, y) & \text{otherwise}.
\end{cases}
$, where $e^{h_{n + 1}} \geq \mu \geq 0$. Then, we can  lower bound the conditional regret of $\sfL_{\rm{RL2D}}$ by using $\Delta \sC_{\sfL_{\rm{RL2D}}, \sH}(h, x) \geq \sC_{\sfL_{\rm{RL2D}}}(h) - \sC^*_{\sfL_{\rm{RL2D}}}(h_{\mu})$ for any $e^{h_{n + 1}} \geq \mu \geq 0$:
\begin{align*}
& \Delta \sC_{\sfL_{\rm{RL2D}}, \sH}(h, x) \\
& \geq \sup_{e^{h_{n + 1}} \geq \mu \geq 0} \paren*{\sC_{\sfL_{\rm{RL2D}}}(h) - \sC^*_{\sfL_{\rm{RL2D}}}(h_{\mu})}\\
& \geq \sup_{e^{h_{n + 1}} \geq \mu \geq 0} \paren[\Bigg]{q_{y_{\max}} \paren*{1 - \frac{e^{h_{y_{\max}}}}{\sum_{y' \in \ov \sY} e^{h_{y'}}}} + \paren*{p_{y_{\max}} - q_{y_{\max}}} \paren*{1 - \frac{e^{h_{y_{\max}}} + e^{h_{n + 1}}}{\sum_{y' \in \ov \sY} e^{h_{y'}}}}\\
& \qquad + \sum_{y' \in \sY, y' \neq y_{\max}} \paren*{p_{y'} - q_{y'}} \paren*{1 - \frac{e^{h_{y'}} + e^{h_{n + 1}}}{\sum_{y' \in \ov \sY} e^{h_{y'}}}}\\
& \qquad - q_{y_{\max}} \paren*{1 - \frac{e^{h_{n + 1}} - \mu}{\sum_{y' \in \ov \sY} e^{h_{y'}}}} - \paren*{p_{y_{\max}} - q_{y_{\max}}} \paren*{1 - \frac{e^{h_{n + 1}} + e^{h_{h_{\max}}}}{\sum_{y' \in \ov \sY} e^{h_{y'}}}}} \\
& \qquad - \sum_{y' \in \sY, y' \neq y_{\max}} \paren*{p_{y'} - q_{y'}} \paren*{1 - \frac{e^{h_{y'}} + e^{h_{y_{\max}}} + \mu}{\sum_{y' \in \ov \sY} e^{h_{y'}}}}\\
& = \frac{1}{\sum_{y' \in \ov \sY} e^{h_{y'}}} \sup_{e^{h_{n + 1}} \geq \mu \geq 0} \paren*{q_{y_{\max}} \paren*{e^{h_{n + 1}} - \mu - e^{h_{y_{\max}}}} + \paren*{p_{n + 1} - p_{y_{\max}} + q_{y_{\max}}}  \paren*{e^{h_{y_{\max}}} + \mu - e^{h_{n + 1}}}}\\
& = \paren*{p_{n + 1}- p_{y_{\max}}} \frac{e^{h_{y_{\max}}}}{\sum_{y' \in \ov \sY} e^{h_{y'}}} \tag{$\mu = e^{h_{n + 1}}$ achieves the maximum}\\
&\geq \frac1{n + 1} \paren*{p_{n + 1}- p_{y_{\max}}}
\tag{by the assumption $ h_{h_{\max}}  = h_{y_{\max}}\geq h_{n + 1} $}\\
& = \frac1{n + 1} \paren*{\Delta \sC_{\ldef, \sH}(h, x)} \tag{by the assumption $p_{n + 1} \geq p_{y_{\max}}$ and $h_{y_{\max}} - h_{n + 1} \geq 0$}
\end{align*}
\end{enumerate}
This proves the inequality \eqref{eq:cond-bound-new}. By Theorem~\ref{Thm:tool-Gamma}, we complete the proof.
\end{proof}

\newpage
\subsection{\texorpdfstring{$\Psi(t) = -\log(t) $}{Psi}}
\label{app:log}

\begin{restatable}{theorem}{HConsistencyNewLog}
\label{Thm:H-consistency-new-log}
Assume that $\sH$ is symmetric and complete. Assume that $c(x, y) =
1_{\expertexpert(x) \neq y}$. Then, for all $h \in \sH$ and any
distribution, the following $\sH$-consistency bound holds:
\begin{equation*}
\sE_{\ldef}(h) - \sE_{\ldef}(\sH) + \sM_{\ldef}(\sH) \leq 2 \sqrt{\sE_{\sfL_{\rm{RL2D}}}(h) - \sE_{\sfL_{\rm{RL2D}}}(\sH) + \sM_{\sfL_{\rm{RL2D}}}(\sH)}.
\end{equation*}
\end{restatable}

\begin{proof}
We can write the conditional error of the surrogate loss as follows:
\begin{align*}
& \sC_{\sfL_{\rm{RL2D}}}(h, x)\\
& = \sum_{y \in \sY} p(x, y) \sfL_{\rm{RL2D}}(h, x, y)\\
& = -\sum_{y \in \sY} p(x, y) c(x, y) \log \paren*{\frac{e^{h(x, y)}}{\sum_{y' \in \ov \sY} e^{h(x, y')}}} -  \sum_{y \in \sY} p(x, y) \paren*{1 - c(x, y)} \log \paren*{\frac{e^{h(x, y)} + e^{h(x, n + 1)}}{\sum_{y' \in \ov \sY} e^{h(x, y')}}}\\
&= -\sum_{y \in \sY} q_y \log \paren*{\frac{e^{h_y}}{\sum_{y' \in \ov \sY} e^{h_{y'}}}} - \sum_{y \in \sY} \paren*{p_y - q_y} \log \paren*{\frac{e^{h_y} + e^{h_{n + 1}}}{\sum_{y' \in \ov \sY} e^{h_{y'}}}}.
\end{align*}
By Lemma~\ref{lemma:delta_target}, the conditional regret of the deferral loss can be expressed as
\begin{equation*}
\Delta \sC_{\ldef, \sH}(h, x) = \max \curl*{p_{y_{\max}}, p_{n + 1}} -  p_{\hh}.
\end{equation*}
Next, we will show that the conditional regret of the surrogate loss can be lower bounded as follows:
\begin{equation}
\label{eq:cond-bound-new-log}
\Delta \sC_{\sfL_{\rm{RL2D}}, \sH}(h, x) = \sC_{\sfL_{\rm{RL2D}}}(h) - \sC^*_{\sfL_{\rm{RL2D}}}(\sH) \geq \frac12 \paren*{\Delta \sC_{\ldef, \sH}(h, x) }^2.
\end{equation}

We first consider the case where $\expertexpert(x) \neq y_{\max}$. Otherwise, it would be straightforward to see that the bound holds. In the case where $\expertexpert(x) \neq y_{\max}$, we have $q_{y_{\max}} = p_{y_{\max}}$.
We first prove that for any hypothesis $h$ and $x \in \sX$, if $y_{\max} \neq h_{\max}$, then the conditional error of $h$ can be lower bounded by that of $\ov h$, which satisfies that
$\ov h(x, y) = \begin{cases}
h_{h_{\max}} &  y = y_{\max}\\
h_{y_{\max}} & y = h_{\max}\\
h_{y} & \text{otherwise}.
\end{cases}
$. Indeed,
\begin{align*}
& \sC_{\sfL_{\rm{RL2D}}}(h) - \sC_{\sfL_{\rm{RL2D}}}(\ov h)\\
& = -q_{y_{\max}} \log \paren*{\frac{e^{h_{y_{\max}}}}{\sum_{y' \in \ov \sY} e^{h_{y'}}}} - \paren*{p_{y_{\max}} - q_{y_{\max}}} \log \paren*{\frac{e^{h_{y_{\max}}} + e^{h_{n + 1}}}{\sum_{y' \in \ov \sY} e^{h_{y'}}}} \\
& \qquad -q_{h_{\max}} \log \paren*{ \frac{e^{h_{h_{\max}}}}{\sum_{y' \in \ov \sY} e^{h_{y'}}}} - \paren*{p_{h_{\max}} - q_{h_{\max}}} \log \paren*{\frac{e^{h_{h_{\max}}} + e^{h_{n + 1}}}{\sum_{y' \in \ov \sY} e^{h_{y'}}}}\\
& \qquad + q_{y_{\max}} \log \paren*{\frac{e^{h_{h_{\max}}}}{\sum_{y' \in \ov \sY} e^{h_{y'}}}} + \paren*{p_{y_{\max}} - q_{y_{\max}}} \log \paren*{\frac{e^{h_{h_{\max}}} + e^{h_{n + 1}}}{\sum_{y' \in \ov \sY} e^{h_{y'}}}} \\
& \qquad + q_{h_{\max}} \log \paren*{ \frac{e^{h_{y_{\max}}}}{\sum_{y' \in \ov \sY} e^{h_{y'}}}} + \paren*{p_{h_{\max}} - q_{h_{\max}}} \log \paren*{\frac{e^{h_{y_{\max}}} + e^{h_{n + 1}}}{\sum_{y' \in \ov \sY} e^{h_{y'}}}}\\
& = \paren*{ q_{y_{\max}} - q_{h_{\max}}} \log\paren*{\frac{e^{h_{h_{\max}}}}{e^{h_{y_{\max}}}}} + \paren*{ p_{y_{\max}} - q_{y_{\max}} - p_{h_{\max}} + q_{h_{\max}} } \log\paren*{\frac{e^{h_{h_{\max}}} + e^{h_{n + 1}}}{e^{h_{y_{\max}}} + e^{h_{n + 1}}}}\\
& \geq \paren*{ p_{y_{\max}} - p_{h_{\max}}} \log\paren*{\frac{e^{h_{h_{\max}}} + e^{h_{n + 1}}}{e^{h_{y_{\max}}} + e^{h_{n + 1}}}}\\
& \geq 0.
\end{align*}
Therefore, we only need to lower bound the conditional regret of hypothesis $h$ satisfying $y_{\max} = h_{\max}$. 
Since $c(x, y) = 1_{\expertexpert(x) \neq y}$, we have $p_{y_{\max}} \geq p_{n + 1} = p_{\expertexpert(x)}$.
Note that when $(p_{y_{\max}} - p_{n + 1}) (h_{y_{\max}} - h_{n + 1}) > 0$, we have $\Delta \sC_{\ldef, \sH}(h, x) = \max \curl*{p_{y_{\max}}, p_{n + 1}} -  p_{\hh} = 0$. When $h_{y_{\max}} - h_{n + 1} \leq 0$, we define a new hypothesis $h_{\mu}$ such that 
    $h_{\mu}(x, y) = \begin{cases}
\log \paren*{e^{h_{n + 1}} + \mu} &  y = y_{\max}\\
\log \paren*{e^{h_{y_{\max}}} - \mu} & y = n + 1\\
h(x, y) & \text{otherwise}.
\end{cases}
$, where $e^{h_{y_{\max}}} - e^{h_{n + 1}} \leq \mu \leq e^{h_{y_{\max}}}$. Then, we can  lower bound the conditional regret of $\sfL_{\rm{RL2D}}$ by using $\Delta \sC_{\sfL_{\rm{RL2D}}, \sH}(h, x) \geq \sC_{\sfL_{\rm{RL2D}}}(h) - \sC^*_{\sfL_{\rm{RL2D}}}(h_{\mu})$ for any $e^{h_{y_{\max}}} - e^{h_{n + 1}} \leq \mu \leq e^{h_{y_{\max}}}$:
\begin{align*}
& \Delta \sC_{\sfL_{\rm{RL2D}}, \sH}(h, x) \\
& \geq \sup_{e^{h_{y_{\max}}} \geq  \mu \geq e^{h_{y_{\max}}} - e^{h_{n + 1}}} \paren*{\sC_{\sfL_{\rm{RL2D}}}(h) - \sC^*_{\sfL_{\rm{RL2D}}}(h_{\mu})}\\
& \geq \sup_{e^{h_{y_{\max}}} \geq  \mu \geq e^{h_{y_{\max}}} - e^{h_{n + 1}}} \paren[\Bigg]{-q_{y_{\max}} \log \paren*{ \frac{e^{h_{y_{\max}}}}{\sum_{y' \in \ov \sY} e^{h_{y'}}}} - \paren*{p_{y_{\max}} - q_{y_{\max}}} \log \paren*{ \frac{e^{h_{y_{\max}}} + e^{h_{n + 1}}}{\sum_{y' \in \ov \sY} e^{h_{y'}}}}\\
& \qquad - \sum_{y' \in \sY, y' \neq y_{\max}} \paren*{p_{y'} - q_{y'}} \log \paren*{ \frac{e^{h_{y'}} + e^{h_{n + 1}}}{\sum_{y' \in \ov \sY} e^{h_{y'}}}}\\
& \qquad + q_{y_{\max}} \log \paren*{ \frac{e^{h_{n + 1}} + \mu}{\sum_{y' \in \ov \sY} e^{h_{y'}}}} + \paren*{p_{y_{\max}} - q_{y_{\max}}} \log \paren*{\frac{e^{h_{n + 1}} + e^{h_{y_{\max}}}}{\sum_{y' \in \ov \sY} e^{h_{y'}}}}} \\
& \qquad + \sum_{y' \in \sY, y' \neq y_{\max}} \paren*{p_{y'} - q_{y'}} \log \paren*{ \frac{e^{h_{y'}} + e^{h_{y_{\max}}} - \mu}{\sum_{y' \in \ov \sY} e^{h_{y'}}}}\\
& = \sup_{e^{h_{y_{\max}}} \geq  \mu \geq e^{h_{y_{\max}}} - e^{h_{n + 1}}} \paren*{q_{y_{\max}} \log \frac{e^{h_{n + 1}} + \mu}{e^{h_{y_{\max}}}} + \sum_{y' \in \sY, y' \neq y_{\max}} \paren*{p_{y'} - q_{y'}}  \log \frac{e^{h_{y'}} + e^{h_{y_{\max}}} - \mu}{e^{h_{y'}} + e^{h_{n + 1}}}}\\
\tag{$e^{h_{y_{\max}}} - e^{h_{n + 1}} \leq \mu \leq e^{h_{y_{\max}}}$}
& \geq \sup_{e^{h_{y_{\max}}} \geq  \mu \geq e^{h_{y_{\max}}} - e^{h_{n + 1}}} \paren*{q_{y_{\max}} \log \frac{e^{h_{n + 1}} + \mu}{e^{h_{y_{\max}}}} + \sum_{y' \in \sY, y' \neq y_{\max}} \paren*{p_{y'} - q_{y'}}  \log \frac{e^{h_{y_{\max}}} - \mu}{e^{h_{n + 1}}}}\\
& = \sup_{e^{h_{y_{\max}}} \geq  \mu \geq e^{h_{y_{\max}}} - e^{h_{n + 1}}} \paren*{q_{y_{\max}} \log \frac{e^{h_{n + 1}} + \mu}{e^{h_{y_{\max}}}} + \paren*{p_{n + 1} - \paren*{p_{y_{\max}} - q_{y_{\max}}}}  \log \frac{e^{h_{y_{\max}}} - \mu}{e^{h_{n + 1}}}}.
\end{align*}
By differentiating with respect to $\mu$, we obtain that
\begin{equation*}
\mu = \frac{q_{y_{\max}} e^{h_{y_{\max}}} - \paren*{p_{n + 1} - \paren*{p_{y_{\max}} - q_{y_{\max}}}} e^{h_{n + 1}} }{q_{y_{\max}} + \paren*{p_{n + 1} - \paren*{p_{y_{\max}} - q_{y_{\max}}}}}
\end{equation*}
achieves the maximum. Plugging it into the expression, we have
\begin{align*}
& \Delta \sC_{\sfL_{\rm{RL2D}}, \sH}(h, x)\\ 
& \geq q_{y_{\max}} \log \bracket*{\frac{\bracket*{e^{h_{y_{\max}}} + e^{ h_{n + 1}}}q_{y_{\max}}}{e^{h_{y_{\max}}}\bracket*{q_{y_{\max}} + \paren*{p_{n + 1} - \paren*{p_{y_{\max}} - q_{y_{\max}}}}}}}\\
& \qquad + \paren*{p_{n + 1} - \paren*{p_{y_{\max}} - q_{y_{\max}}}}\log\bracket*{\frac{\bracket*{e^{h_{y_{\max}}} + e^{ h_{n + 1}}}\paren*{p_{n + 1} - \paren*{p_{y_{\max}} - q_{y_{\max}}}}}{ e^{ h_{n + 1}}\bracket*{q_{y_{\max}}+\paren*{p_{n + 1} - \paren*{p_{y_{\max}} - q_{y_{\max}}}}}}}.
\end{align*}
This can be further lower bounded by taking the minimum over $h \in \sH$, where the minimum is attained when $e^{h_{y_{\max}}} = e^{ h_{n + 1}}$ Therefore,
\begin{align*}
& \Delta \sC_{\sfL_{\rm{RL2D}}, \sH}(h, x)\\
& \geq q_{y_{\max}} \log \bracket*{\frac{2q_{y_{\max}}}{q_{y_{\max}} + \paren*{p_{n + 1} - \paren*{p_{y_{\max}} - q_{y_{\max}}}}}}\\
& \qquad + \paren*{p_{n + 1} - \paren*{p_{y_{\max}} - q_{y_{\max}}}}\log\bracket*{\frac{2\paren*{p_{n + 1} - \paren*{p_{y_{\max}} - q_{y_{\max}}}}}{ q_{y_{\max}}+\paren*{p_{n + 1} - \paren*{p_{y_{\max}} - q_{y_{\max}}}}}}.
\end{align*}
By applying Pinsker’s inequality \citep[Proposition~E.7]{MohriRostamizadehTalwalkar2018}, we obtain
\begin{align*}
& \Delta \sC_{\sfL_{\rm{RL2D}}, \sH}(h, x)\\ 
& \geq \bracket*{q_{y_{\max}}+p_{n + 1} - \paren*{p_{y_{\max}} - q_{y_{\max}}}}\\
& \qquad \times \frac12\bracket*{ \abs*{\frac{q_{y_{\max}}}{q_{y_{\max}}+p_{n + 1} - \paren*{p_{y_{\max}} - q_{y_{\max}}}}-\frac12}+\abs*{\frac{p_{n + 1} - \paren*{p_{y_{\max}} - q_{y_{\max}}}}{q_{y_{\max}}+p_{n + 1} - \paren*{p_{y_{\max}} - q_{y_{\max}}}}-\frac12}}^2\\
& \geq \frac{1}{2} \frac{\paren*{p_{y_{\max}} - p_{n + 1}}^2}{q_{y_{\max}}+p_{n + 1} - \paren*{p_{y_{\max}} - q_{y_{\max}}}}\\
& \geq \frac12 \paren*{p_{y_{\max}} - p_{n + 1}}^2  
\tag{$q_{y_{\max}}+p_{n + 1} - \paren*{p_{y_{\max}} - q_{y_{\max}}} \leq 1$}
\\
& = \frac12 \paren*{\Delta \sC_{\ldef, \sH}(h, x)}^2 \tag{by the assumption $p_{y_{\max}} \geq p_{n + 1}$ and $h_{y_{\max}} \leq h_{n + 1}$}
\end{align*}
This proves the inequality \eqref{eq:cond-bound-new-log}.
In the case where $\expertexpert(x) = y_{\max}$, we have $p_{n + 1} = p_{y_{\max}}$. By Lemma~\ref{lemma:delta_target}, the conditional regret of the deferral loss can be expressed as
$
\Delta \sC_{\ldef, \sH}(h, x) = p_{n + 1} -  p_{\hh}.
$
If $\hh(x) = n + 1$, then we have $\Delta \sC_{\ldef, \sH}(h, x) = 0$. Otherwise, when $\hh(x) \neq n + 1$, we can proceed in the similar way as above, by defining a new hypothesis $h_{\mu}$ such that
    $h_{\mu}(x, y) = \begin{cases}
\log \paren*{e^{h_{n + 1}} + \mu} &  y = \hh(x)\\
\log \paren*{e^{h_{\hh(x)}} - \mu} & y = n + 1\\
h(x, y) & \text{otherwise}
\end{cases}
$. Then, we can  lower bound the conditional regret of $\sfL_{\rm{RL2D}}$ by using $\Delta \sC_{\sfL_{\rm{RL2D}}, \sH}(h, x) \geq \sC_{\sfL_{\rm{RL2D}}}(h) - \sC^*_{\sfL_{\rm{RL2D}}}(h_{\mu})$, by applying the same derivation as above, modulo replacing $y_{\max}$ with $\hh(x)$. This leads to the inequality \eqref{eq:cond-bound-new-log} as well. By Theorem~\ref{Thm:tool-Gamma}, we complete the proof.
\end{proof}

\newpage
\subsection{\texorpdfstring{$\Psi(t) = \frac{1}{q} \paren*{1 - t^{q}}$}{Psi}}
\label{app:gce}

\begin{restatable}{theorem}{HConsistencyNewGCE}
\label{Thm:H-consistency-new-gce}
Assume that $\sH$ is symmetric and complete. Assume that $c(x, y) = 1_{\expertexpert(x) \neq y}$. Then, for all $h \in \sH$ and any distribution, the following $\sH$-consistency bound holds:
\begin{equation*}
\sE_{\ldef}(h) - \sE_{\ldef}(\sH) + \sM_{\ldef}(\sH) \leq 2 \sqrt{(n + 1)^{\alpha} \paren*{\sE_{\sfL_{\rm{RL2D}}}(h) - \sE_{\sfL_{\rm{RL2D}}}(\sH) + \sM_{\sfL_{\rm{RL2D}}}(\sH)} }.
\end{equation*}
\end{restatable}

\begin{proof}
We can write the conditional error of the surrogate loss as follows:
\begin{align*}
\sC_{\sfL_{\rm{RL2D}}}(h, x)
& =  \sum_{y \in \sY} p(x, y) \sfL_{\rm{RL2D}}(h, x, y)\\
& = \frac{1}{q} \sum_{y \in \sY} p(x, y) c(x, y) \paren*{1- \paren*{\frac{e^{h(x, y)}}{\sum_{y' \in \ov \sY} e^{h(x, y')}}}^q}\\
& \qquad + \frac{1}{q} \sum_{y \in \sY} p(x, y) \paren*{1 - c(x, y)} \paren*{1 - \paren*{\frac{e^{h(x, y)} + e^{h(x, n + 1)}}{\sum_{y' \in \ov \sY} e^{h(x, y')}}}^q}\\
&= \frac{1}{q} \sum_{y \in \sY} q_y\paren*{ 1- \paren*{\frac{e^{h_y}}{\sum_{y' \in \ov \sY} e^{h_{y'}}}}^q} + \frac{1}{q}  \sum_{y \in \sY} \paren*{p_y - q_y} \paren*{ 1- \paren*{\frac{e^{h_y} + e^{h_{n + 1}}}{\sum_{y' \in \ov \sY} e^{h_{y'}}}}^q}.
\end{align*}
By Lemma~\ref{lemma:delta_target}, the conditional regret of the deferral loss can be expressed as
\begin{equation*}
\Delta \sC_{\ldef, \sH}(h, x) = \max \curl*{p_{y_{\max}}, p_{n + 1}} -  p_{\hh}.
\end{equation*}
Next, we will show that the conditional regret of the surrogate loss can be lower bounded as follows:
\begin{equation}
\label{eq:cond-bound-new-gce}
\Delta \sC_{\sfL_{\rm{RL2D}}, \sH}(h, x)  = \sC_{\sfL_{\rm{RL2D}}}(h) - \sC^*_{\sfL_{\rm{RL2D}}}(\sH) \geq \frac1{2(n + 1)^q} \paren*{\Delta \sC_{\ldef, \sH}(h, x) }^2.
\end{equation}
We first consider the case where $\expertexpert(x) \neq y_{\max}$. Otherwise, it would be straightforward to see that the bound holds. In the case where $\expertexpert(x) \neq y_{\max}$, we have $q_{y_{\max}} = p_{y_{\max}}$.
We first prove that for any hypothesis $h$ and $x \in \sX$, if $y_{\max} \neq h_{\max}$, then the conditional error of $h$ can be lower bounded by that of $\ov h$, which satisfies that
$\ov h(x, y) = \begin{cases}
h_{h_{\max}} &  y = y_{\max}\\
h_{y_{\max}} & y = h_{\max}\\
h_{y} & \text{otherwise}.
\end{cases}
$. Indeed,
\begin{align*}
& q \paren*{\sC_{\sfL_{\rm{RL2D}}}(h) - \sC_{\sfL_{\rm{RL2D}}}(\ov h)}\\
& = q_{y_{\max}}  \paren*{1 -\paren*{\frac{e^{h_{y_{\max}}}}{\sum_{y' \in \ov \sY} e^{h_{y'}}}}^q} + \paren*{p_{y_{\max}} - q_{y_{\max}}} \paren*{1 - \paren*{\frac{e^{h_{y_{\max}}} + e^{h_{n + 1}}}{\sum_{y' \in \ov \sY} e^{h_{y'}}}}^q} \\
& \qquad + q_{h_{\max}} \paren*{1 - \paren*{ \frac{e^{h_{h_{\max}}}}{\sum_{y' \in \ov \sY} e^{h_{y'}}}}^q} + \paren*{p_{h_{\max}} - q_{h_{\max}}} \paren*{1 - \paren*{\frac{e^{h_{h_{\max}}} + e^{h_{n + 1}}}{\sum_{y' \in \ov \sY} e^{h_{y'}}}}^q}\\
& \qquad - q_{y_{\max}} \paren*{1 - \paren*{\frac{e^{h_{h_{\max}}}}{\sum_{y' \in \ov \sY} e^{h_{y'}}}}^q} - \paren*{p_{y_{\max}} - q_{y_{\max}}} \paren*{1 - \paren*{\frac{e^{h_{h_{\max}}} + e^{h_{n + 1}}}{\sum_{y' \in \ov \sY} e^{h_{y'}}}}^q} \\
& \qquad - q_{h_{\max}} \paren*{1 - \paren*{ \frac{e^{h_{y_{\max}}}}{\sum_{y' \in \ov \sY} e^{h_{y'}}}}^q} + \paren*{p_{h_{\max}} - q_{h_{\max}}} \paren*{1 - \paren*{\frac{e^{h_{y_{\max}}} + e^{h_{n + 1}}}{\sum_{y' \in \ov \sY} e^{h_{y'}}}}^q}\\
& = \paren*{ q_{y_{\max}} - q_{h_{\max}}} \bracket*{\paren*{1 -\paren*{\frac{e^{h_{y_{\max}}}}{\sum_{y' \in \ov \sY} e^{h_{y'}}}}^q} - \paren*{1 - \paren*{\frac{e^{h_{h_{\max}}}}{\sum_{y' \in \ov \sY} e^{h_{y'}}}}^q}}\\
& \qquad + \paren*{ p_{y_{\max}} - q_{y_{\max}} - p_{h_{\max}} + q_{h_{\max}} } \bracket*{\paren*{1 - \paren*{\frac{e^{h_{y_{\max}}} + e^{h_{n + 1}}}{\sum_{y' \in \ov \sY} e^{h_{y'}}}}^q} - \paren*{1 - \paren*{\frac{e^{h_{h_{\max}}} + e^{h_{n + 1}}}{\sum_{y' \in \ov \sY} e^{h_{y'}}}}^q}}\\
& \geq \paren*{ p_{y_{\max}} - p_{h_{\max}}} \bracket*{\paren*{1 - \paren*{\frac{e^{h_{y_{\max}}} + e^{h_{n + 1}}}{\sum_{y' \in \ov \sY} e^{h_{y'}}}}^q} - \paren*{1 - \paren*{\frac{e^{h_{h_{\max}}} + e^{h_{n + 1}}}{\sum_{y' \in \ov \sY} e^{h_{y'}}}}^q}}\\
& \geq 0.
\end{align*}
Therefore, we only need to lower bound the conditional regret of hypothesis $h$ satisfying $y_{\max} = h_{\max}$. 
Since $c(x, y) = 1_{\expertexpert(x) \neq y}$, we have $p_{y_{\max}} \geq p_{n + 1} = p_{\expertexpert(x)}$.
Note that when $(p_{y_{\max}} - p_{n + 1}) (h_{y_{\max}} - h_{n + 1}) > 0$, we have $\Delta \sC_{\ldef, \sH}(h, x) = \max \curl*{p_{y_{\max}}, p_{n + 1}} -  p_{\hh} = 0$. When $h_{y_{\max}} - h_{n + 1} \leq 0$, we define a new hypothesis $h_{\mu}$ such that 
    $h_{\mu}(x, y) = \begin{cases}
\log \paren*{e^{h_{n + 1}} + \mu} &  y = y_{\max}\\
\log \paren*{e^{h_{y_{\max}}} - \mu} & y = n + 1\\
h(x, y) & \text{otherwise}.
\end{cases}
$, where $e^{h_{y_{\max}}} - e^{h_{n + 1}} \leq \mu \leq e^{h_{y_{\max}}}$. Then, we can  lower bound the conditional regret of hypothesis $h$ by using $\Delta \sC_{\sfL_{\rm{RL2D}}, \sH}(h, x) \geq \sC_{\sfL_{\rm{RL2D}}}(h) - \sC^*_{\sfL_{\rm{RL2D}}}(h_{\mu})$ for any  $e^{h_{y_{\max}}} - e^{h_{n + 1}} \leq \mu \leq e^{h_{y_{\max}}}$:
\begin{align*}
& \Delta \sC_{\sfL_{\rm{RL2D}}, \sH}(h, x) \\
& \geq \sup_{e^{h_{y_{\max}}} \geq  \mu \geq e^{h_{y_{\max}}} - e^{h_{n + 1}}} \paren*{\sC_{\sfL_{\rm{RL2D}}}(h) - \sC^*_{\sfL_{\rm{RL2D}}}(h_{\mu})}\\
& \geq \frac1q \sup_{e^{h_{y_{\max}}} \geq  \mu \geq e^{h_{y_{\max}}} - e^{h_{n + 1}}} \paren[\Bigg]{q_{y_{\max}} \paren*{1 - \paren*{ \frac{e^{h_{y_{\max}}}}{\sum_{y' \in \ov \sY} e^{h_{y'}}}}^q} + \paren*{p_{y_{\max}} - q_{y_{\max}}} \paren*{1 - \paren*{ \frac{e^{h_{y_{\max}}} + e^{h_{n + 1}}}{\sum_{y' \in \ov \sY} e^{h_{y'}}}}^q}\\
& \qquad + \sum_{y' \in \sY, y' \neq y_{\max}} \paren*{p_{y'} - q_{y'}} \paren*{1 - \paren*{ \frac{e^{h_{y'}} + e^{h_{n + 1}}}{\sum_{y' \in \ov \sY} e^{h_{y'}}}}^q}\\
& \qquad - q_{y_{\max}} \paren*{1 - \paren*{ \frac{e^{h_{n + 1}} + \mu}{\sum_{y' \in \ov \sY} e^{h_{y'}}}}^q} - \paren*{p_{y_{\max}} - q_{y_{\max}}} \paren*{1 - \paren*{\frac{e^{h_{n + 1}} + e^{h_{y_{\max}}}}{\sum_{y' \in \ov \sY} e^{h_{y'}}}}}^q \\
& \qquad - \sum_{y' \in \sY, y' \neq y_{\max}} \paren*{p_{y'} - q_{y'}} \paren*{1 - \paren*{ \frac{e^{h_{y'}} + e^{h_{y_{\max}}} - \mu}{\sum_{y' \in \ov \sY} e^{h_{y'}}}}^q}}\\
& \geq \frac1q \sup_{e^{h_{y_{\max}}} \geq  \mu \geq e^{h_{y_{\max}}} - e^{h_{n + 1}}} \paren[\Bigg]{q_{y_{\max}} \paren*{1 - \paren*{ \frac{e^{h_{y_{\max}}}}{\sum_{y' \in \ov \sY} e^{h_{y'}}}}^q} + \sum_{y' \in \sY, y' \neq y_{\max}} \paren*{p_{y'} - q_{y'}} \paren*{1 - \paren*{ \frac{e^{h_{n + 1}}}{\sum_{y' \in \ov \sY} e^{h_{y'}}}}^q}\\
& \qquad - q_{y_{\max}} \paren*{1 - \paren*{ \frac{e^{h_{n + 1}} + \mu}{\sum_{y' \in \ov \sY} e^{h_{y'}}}}^q} - \sum_{y' \in \sY, y' \neq y_{\max}} \paren*{p_{y'} - q_{y'}} \paren*{1 - \paren*{ \frac{e^{h_{y_{\max}}} - \mu}{\sum_{y' \in \ov \sY} e^{h_{y'}}}}^q} }\tag{$e^{h_{y_{\max}}} - e^{h_{n + 1}} \leq \mu \leq e^{h_{y_{\max}}}$}\\
& =  \frac1q \sup_{e^{h_{y_{\max}}} \geq  \mu \geq e^{h_{y_{\max}}} - e^{h_{n + 1}}} \paren[\Bigg]{q_{y_{\max}} \paren*{1 - \paren*{ \frac{e^{h_{y_{\max}}}}{\sum_{y' \in \ov \sY} e^{h_{y'}}}}^q} + \paren*{p_{n + 1} - \paren*{p_{y_{\max}} - q_{y_{\max}}}} \paren*{1 - \paren*{ \frac{e^{h_{n + 1}}}{\sum_{y' \in \ov \sY} e^{h_{y'}}}}^q}\\
& \qquad - q_{y_{\max}} \paren*{1 - \paren*{ \frac{e^{h_{n + 1}} + \mu}{\sum_{y' \in \ov \sY} e^{h_{y'}}}}^q} - \paren*{p_{n + 1} - \paren*{p_{y_{\max}} - q_{y_{\max}}}} \paren*{1 - \paren*{ \frac{e^{h_{y_{\max}}} - \mu}{\sum_{y' \in \ov \sY} e^{h_{y'}}}}^q} }
\end{align*}
By differentiating with respect to $\mu$, we obtain that
\begin{equation*}
\mu = \frac{\paren*{p_{n + 1} - \paren*{p_{y_{\max}} - q_{y_{\max}}}}^{\frac1{q - 1}} e^{h_{y_{\max}}} -  \paren*{q_{y_{\max}}}^{\frac1{q - 1}} e^{h_{n + 1}} }{\paren*{q_{y_{\max}}}^{\frac1{q - 1}} + \paren*{p_{n + 1} - \paren*{p_{y_{\max}} - q_{y_{\max}}}}^{\frac1{q - 1}}}    
\end{equation*}
achieves the maximum. Plugging it into the expression, we have
\begin{align*}
&\Delta \sC_{\sfL_{\rm{RL2D}}, \sH}(h, x)\\
& \geq \frac1q \paren[\Bigg]{ - q_{y_{\max}} \paren*{ \frac{e^{h_{y_{\max}}}}{\sum_{y' \in \ov \sY} e^{h_{y'}}}}^q - \paren*{p_{n + 1} - \paren*{p_{y_{\max}} - q_{y_{\max}}}} \paren*{ \frac{e^{h_{n + 1}}}{\sum_{y' \in \ov \sY} e^{h_{y'}}}}^q\\
& \qquad + q_{y_{\max}}\bracket*{\frac{\bracket*{e^{ h_{y_{\max}}} + e^{ h_{n + 1}}}\paren*{p_{n + 1} - \paren*{p_{y_{\max}} - q_{y_{\max}}}}^{\frac1{q - 1}}}{\sum_{y' \in \ov \sY} e^{h_{y'}}\bracket*{q_{y_{\max}}^{\frac1{q - 1}}+\paren*{p_{n + 1} - \paren*{p_{y_{\max}} - q_{y_{\max}}}}^{\frac1{q - 1 }}}}}^{q }\\
& \quad \qquad + \paren*{p_{n + 1} - \paren*{p_{y_{\max}} - q_{y_{\max}}}}\bracket*{\frac{\bracket*{e^{ h_{y_{\max}}} + e^{ h_{n + 1}}}q_{y_{\max}}^{\frac1{q - 1}}}{\sum_{y' \in \ov \sY} e^{h_{y'}}\bracket*{q_{y_{\max}}^{\frac1{q - 1 }} + \paren*{p_{n + 1} - \paren*{p_{y_{\max}} - q_{y_{\max}}}}^{\frac1{q - 1}}}}}^{q } }   
\end{align*}
This can be further lower bounded by taking the minimum over $h \in \sH$, where the minimum is attained when $e^{h_{n + 1}} = e^{h_{y_{\max}}} = e^{h_y}$ for all $y \in \sY$. Therefore,
\begin{align*}
\Delta \sC_{\sfL_{\rm{RL2D}}, \sH}(h, x) \geq 
& \frac{2}{(n + 1)^{q }}\paren*{\bracket*{\frac{q_{y_{\max}}^{\frac1{1- q}}+\paren*{p_{n + 1} - \paren*{p_{y_{\max}} - q_{y_{\max}}}}^{\frac1{1- q}}}{2}}^{1- q} - \frac{p_{n + 1} - p_{y_{\max}}}{2}}\\
\tag{minimum is attained when $e^{ h_{n + 1}} = e ^{ h_{y_{\max}}} = e^{h_{y}}, \forall y \in \sY$}\\
& \geq \frac1{2 (n + 1)^{q }} \paren*{p_{y_{\max}} - p_{n + 1}}^2\\
\tag{$q_{y_{\max}}+\paren*{p_{n + 1} - \paren*{p_{y_{\max}} - q_{y_{\max}}}}\leq 1$ and by analyzing the Taylor expansion}\\
& = \frac1{2 (n + 1)^{q }} \paren*{\Delta \sC_{\ldef, \sH}(h, x)}^2 \tag{$p_{y_{\max}} \geq p_{n + 1}$ and $h_{y_{\max}} \leq h_{n + 1} $}
\end{align*}
This proves the inequality \eqref{eq:cond-bound-new-gce}. In the case where $\expertexpert(x) = y_{\max}$, we have $p_{n + 1} = p_{y_{\max}}$. By Lemma~\ref{lemma:delta_target}, the conditional regret of the deferral loss can be expressed as
$
\Delta \sC_{\ldef, \sH}(h, x) = p_{n + 1} -  p_{\hh}.
$
If $\hh(x) = n + 1$, then we have $\Delta \sC_{\ldef, \sH}(h, x) = 0$. Otherwise, when $\hh(x) \neq n + 1$, we can proceed in the similar way as above, by defining a new hypothesis $h_{\mu}$ such that
    $h_{\mu}(x, y) = \begin{cases}
\log \paren*{e^{h_{n + 1}} + \mu} &  y = \hh(x)\\
\log \paren*{e^{h_{\hh(x)}} - \mu} & y = n + 1\\
h(x, y) & \text{otherwise}
\end{cases}
$. Then, we can  lower bound the conditional regret of $\sfL_{\rm{RL2D}}$ by using $\Delta \sC_{\sfL_{\rm{RL2D}}, \sH}(h, x) \geq \sC_{\sfL_{\rm{RL2D}}}(h) - \sC^*_{\sfL_{\rm{RL2D}}}(h_{\mu})$, by applying the same derivation as above, modulo replacing $y_{\max}$ with $\hh(x)$. This leads to the inequality \eqref{eq:cond-bound-new-gce} as well. By Theorem~\ref{Thm:tool-Gamma}, we complete the proof.
\end{proof}

\section{Proof of Theorem~\ref{Thm:minimizability}}
\label{app:minimizability}

\Minimizability*
\begin{proof}
By definition and the Lebesgue dominated convergence theorem, we have
\begin{equation*}
\sM_{\ell_{\rm{comp}}}(\sH) \leq \sE^*_{\ell_{\rm{comp}}}(\sH) \leq \lim_{\alpha \to \plus \infty} \E\bracket*{\Psi \paren*{\frac{e^{\alpha h^*(x, y)}}{\sum_{y' \in \sY} e^{\alpha h^*(x, y')}}}} = 0.  
\end{equation*}
This completes the proof.
\end{proof}

\section{Future work}
\label{app:future_work}

While we presented a comprehensive study of surrogate loss functions for learning to defer, our work focused on the standard single-expert and single-stage setting, aligning with previous work \citep{pmlr-v206-mozannar23a}. However, an interesting direction is to extend our approach to multi-expert \citep{verma2023learning} and two-stage settings \citep{MaoMohriMohriZhong2023two}, which we have left for future work.

\end{document}